\documentclass{article}
\usepackage{fullpage}
\usepackage{microtype}
\usepackage{graphicx}
\usepackage{subfigure}
\usepackage{booktabs}
\usepackage{mathrsfs}
\usepackage{multirow}
\usepackage{blkarray}
\usepackage{bm}
\usepackage{bbm}
\usepackage{algorithm}
\usepackage{algorithmic}
\usepackage{color}
\usepackage[numbers]{natbib}
\usepackage{mathrsfs}
\usepackage{amsmath} 
\usepackage{float}
\usepackage{amssymb}
\usepackage{amsthm}
\usepackage{enumitem}
\usepackage{mathtools}
\newtheorem{thm}{Theorem}
\newtheorem{lemma}{Lemma}
\theoremstyle{definition}

\newtheorem{assume}{Assumption}
\newtheorem{remark}{Remark}
\newtheorem{corollary}{Corollary}
\newtheorem{prop}{Proposition}
\usepackage{tikz}
\usepackage{hyperref}
\author{Dachao Lin 
	\thanks{Academy for Advanced Interdisciplinary Studies,
		Peking University.
		\texttt{lindachao@pku.edu.cn}}
	\and
	Zhihua Zhang 
	\thanks{School of Mathematical Sciences,
		Peking University.
		\texttt{zhzhang@math.pku.edu.cn}}
}

\title{Directional Convergence Analysis under \\
	Spherically Symmetric Distribution}

\begin{document}

\maketitle

\begin{abstract}
	We consider the fundamental problem of learning  linear predictors (i.e., separable datasets with zero margin) using neural networks with gradient flow or gradient descent. 
	Under the assumption of spherically symmetric data distribution,
	we show directional convergence guarantees with exact convergence rate for two-layer non-linear networks with only two hidden nodes, and (deep) linear networks. 
	Moreover, our discovery is built on dynamic from the initialization without both initial loss and perfect classification constraint in contrast to previous works. We also point out and study the challenges in further strengthening and generalizing our results.
\end{abstract}
	
\section{Introduction}
	In recent years, deep neural networks have been successfully trained with simple gradient-based methods, despite the inherent non-convexity of the learning problem.
	Meanwhile, implicit biases introduced by optimization algorithms play a crucial role in training neural networks. 
	Previous work rigorously analyzes the optimization of deep networks, leading to many exciting developments such as the neural tangent kernel \cite{jacot2018neural, du2019gradient, arora2019fine, allen2019convergence, zou2018stochastic}. The above works are usually based on over-parameterization regime, which makes the weights stay close to initialization and share similar dynamic with linear regression. 
	By contrast, \citet{lyu2019gradient} showed that if the data can be perfectly classified, the parameters are guaranteed to diverge in norm to infinity, meaning that the prediction surface can continually change during the training in classification tasks. 
	Moreover, \citet{ji2020directional}  experimentally illustrated that even on simple data, the prediction surface continues to change after perfect classification is achieved, and even with large width it is not close to the maximum margin predictor from the neural tangent regime. These findings raise an issue that the properties about neural networks may be unstable if the prediction surface never stops changing.

	\citet{lyu2019gradient, ji2020directional, ji2018gradient} already addressed this issue by guaranteeing stable (directional) convergence behavior of deep networks as training proceeds, despite the growth of weight vectors to infinity. 
	Their works focus on general homogeneous deep networks and exponential-type loss functions at the ``late training'' phase, meaning that the predictor has already obtained zero classification error. 
	Moreover, they only proved asymptotic directional convergence, i.e., the parameters converge in direction (to a KKT point of a natural max-margin problem), and used a finite data case to ensure positive maximum margin. But the traditional dataset is large-scale and may be inseparable or separable with very small margin. 
	In this paper, we consider a more specific case for binary classification under population logit loss (binary cross-entropy) with zero margin. For analysis simplicity
	we make the assumption that the distribution of input data is spherically symmetric,  including the standard Gaussian as a special case. 
	Furthermore, the networks we consider include (deep) linear networks and a two-layer non-linear network with only two hidden nodes and fixed second layer. 
	The paper and the main contributions are organized as follows:
	
	\begin{itemize}
		\item We present the analysis of linear networks (linear predictors) in Section \ref{sec:shallow-linear} and show that gradient flow gives logarithmic directional convergence for any initialization. In addition, we find that gradient descent provides same convergence rate if the weights are initialized with certain large or small norm under bounded learning rates sequence. Both the findings show improved convergence bound compared to general empirical loss and positive margin setting as we expected, because we make the benign distribution assumption.
		\item In Section \ref{sec:deep-linear} we build up deep linear networks motivated by the previous results, showing at least two-phase convergence rates during optimization with the gradient flow method. 
		\item In Section \ref{sec:shallow-non-linear} we discuss a two-layer nonlinear network with only two hidden neurons and fixed different signs of second layer weights, giving the simplest recovery guarantee in our setting. We find positive alignment in direction for half the possible initialization. Built on this result we give exact logarithmic directional convergence for ReLU activation under gradient flow, as well as gradient descent methods with constrained learning rates.
		\item Motivated by the above results, we conduct several experiments in Section \ref{sec:exp} for potential improvements of our current results, particularly some bad initialization that gives slow convergence, and the remaining complex initialization case for non-linear networks can be viewed as future work. 
	\end{itemize}
	In summary, we hope our work contributes to a better understanding of the optimization dynamics of gradient methods on deep neural networks in classification tasks.

\subsection{Related Work}
There is a rapid growth of literature on analyzing the directional convergence, as well as the optimization dynamic of neural network objectives, surveying all of which is well outside our scope. Thus, we only briefly survey the works most related to ours.

\paragraph{Optimization Dynamics of Neural Networks.} 
A large amount of works focus on training networks under regression case with square loss. 
\citet{yehudai2020learning} analyze a single neuron in a realizable setting under general families of input distributions and activations, showing that some assumptions on both the distribution and the activation function are necessary, and proving linear convergence under mild assumptions.
\citet{tian2019luck, yu2019student} analyze one-hidden-layer NNs with multi-nodes and ReLU activation under over-parameterization setting (more teacher nodes than student nodes), showing local convergence with assumption of small overlapping teacher node.
These recovery tasks do not suffer infinite norm problem but with more difficult learning requirements including direction as well as the norm of teacher weights.

There also have plenty of literature focus on classification task under cross entropy loss or logit loss. \citet{li2018learning} consider learning a two-layer over-parameterized ReLU neural network with cross entropy loss via stochastic gradient descent, the main technique they adopted is that the most activation pattern is set to be unchanged as the initialization, but we could hardly use since for continuous data distribution, any change of weight would effect the activation pattern for some data. 
\citet{cao2019generalization,Cao_2020} also analyze cross-entropy loss with over-parameterized networks, similarly, the parameter range they focus on is in a small neighborhood of initialization as most over-parameterized work do. 
We analyze binary cross entropy loss in really under-parameterized settings (two neurons or linear activation), to discover the concrete variation during training.
Thanks to the theories of previous works, we have more and more profound understanding of deep learning.

\paragraph{Directional Convergence under Separable Data.} 
Separated data can be dated back to \citet{freund1997decision} from the literature on AdaBoost. A recent close line of work is an analysis of gradient descent for logistic regression when the data is separable \cite{soudry2018implicit, nacson2019stochastic, nacson2019convergence, ji2018risk}. 
These works consider linear classifiers with smooth monotone loss functions including the cross-entropy loss, optimized on linearly separable data similar as us but with finite data as realistic setting with bounded learning rates.
Directional convergence has appeared in the literature \cite{gunasekar2018implicit, chizat2020implicit}, and established for linear predictors \cite{soudry2018implicit, nacson2019convergence, nacson2019stochastic, ji2020gradient, shamir2020gradient}.
\citet{ji2018gradient, ji2020directional, lyu2019gradient} extend linear classifiers to deep homogeneous networks using powerful techniques. 
The findings build on the alignment of some weights of neural network reaching a stationary point of the limiting margin maximization objective under the gradient methods. 
We consider the objective with logit loss and good input distribution to provide the potential directional convergence rate under population loss, and obtain convergence rate and the analysis of the whole training period.

\section{Preliminaries}
	In this paper, we would learn  predictor $\phi(\cdot, \mathbf{w})\colon \mathbb{R}^d \to \mathbb{R}$ with parameters $\mathbf{w}$, 
	and let $ \mathbf{x} \in \mathbb{R}^d $ be input features sampled from some unknown distribution $\mathcal{D}$ with a well-defined covariance matrix. We consider a binary classification
	problem in which the label for each $\mathbf{x}$ is decided by such a normalized vector $\mathbf{v} \in \mathbb{R}^d$ that $\|\mathbf{v}\|=1$, and in this way $y(\mathbf{x})=\text{sgn}(\mathbf{v}^\top\mathbf{x}) \in\{{-}1, {+}1\}$\footnote{Here, $\text{sgn}(z)=1$ if $z\geq0$, otherwise $-1$.}.
 We employ the population logit loss $ \ell(z) := \ln(1 + e^{-z}) $ (binary cross-entropy) with $z(\mathbf{x}, \mathbf{w})=y(\mathbf{x})\phi(\mathbf{x}, \mathbf{w})$.  Thus, the objective is
	\begin{equation*}
	\min_{\mathbf{w}} L(\mathbf{w}) :=\mathbb{E}_{\mathbf{x}\sim\mathcal{D}}\ln\left(1+e^{-y(\mathbf{x})\phi(\mathbf{x}, \mathbf{w})}\right).
	\end{equation*}	
	We focus on the following standard gradient methods applied with $\nabla L(\mathbf{w})$:
	\begin{itemize}
		\item \textbf{Gradient flow}: We initialize  $\mathbf{w}(0)$, and for every $ t>0 $ we let $\mathbf{w}(t)$  be the solution of the
		differential equation: $ \dot{\mathbf{w}}(t) = -\nabla L(\mathbf{w}(t)) $.
		\item \textbf{Gradient descent}: We initialize  $\mathbf{w}(0)$ and set a sequence of positive learning rates  $\{\eta_n\}_{n=1}^\infty$. At each iteration $ t > 0 $,
		we do a single step in the negative direction of the gradient, that is, $\mathbf{w}(n+1)=\mathbf{w}(n)-\eta_n\nabla L(\mathbf{w}(n))$.
	\end{itemize}

	\paragraph{Notation.}  In this paper $\|\cdot\|$ denotes the $ \ell_2 $-norm, $\Sigma = \mathbb{E}_{\mathbf{x}\sim\mathcal{D}} \mathbf{x} \mathbf{x}^{\top}$ is the population covariance matrix, and $\lambda_{\max}(A)$ is the maximum eigenvalue of the real symmetric matrix $A$.  
	We let $\overline{\mathbf{w}} := \frac{\mathbf{w}}{\|\mathbf{w}\|}$ whenever $\|\mathbf{w}\|\neq 0$. Given vectors $\mathbf{w}$ and $\mathbf{v}$, we let $ \theta(\mathbf{w}, \mathbf{v}) := \arccos\left(\frac{\mathbf{w}^\top\mathbf{v}}{\|\mathbf{w}\|\|\mathbf{v}\|}\right)\in[0, \pi] $ denote the angle between $\mathbf{w}$ and $\mathbf{v}$,  $\theta(t):=\theta(\mathbf{w}(t), \mathbf{v})$ and  $\theta(n):=\theta(\mathbf{w}(n), \mathbf{v})$. 
	Let $ \mathcal{D}_{\mathbf{w},\mathbf{v}} $ be the marginal distribution of $\mathbf{x}$ on the subspace spanned by $\mathbf{w}, \mathbf{v}$ (a distribution over $\mathbb{R}^2$), $\mathcal{D}_2:=\mathcal{D}_{\mathbf{e}_1,\mathbf{e}_2}$ and  $c_0:=\mathbb{E}_{\mathbf{x}\sim \mathcal{D}_2}\|\mathbf{x}\|$. We
	 call complexity $\mathcal{O}\left(\ln^\alpha(1/\epsilon)\right)$ for $\alpha>0$ to obtain $\epsilon$-error  \textit{logarithmic} convergence.

\section{Shallow Linear Networks} \label{sec:shallow-linear}
In this section, we begin with the case of classical linear predictors: $\phi(\mathbf{x}, \mathbf{w})=\mathbf{w}^\top\mathbf{x}$, and defer the  proofs to Appendix \ref{app:shallow-linear}. Noting that 
$ \nabla L(\mathbf{w}) = -\mathbb{E}_{\mathbf{x}\sim\mathcal{D}} \frac{y(\mathbf{x}) \mathbf{x} }{1+e^{y(\mathbf{x})\mathbf{w}^{\top}\mathbf{x}}} $, 
we have the following basic properties.

\begin{prop}\label{prop:basic}
 The objective function	$L(\mathbf{w})$ is $c$-Lipschitz continuous,  $\frac{1}{4}\lambda_{\max}(\Sigma)$-smooth, and convex but not strongly convex if $P \{ \mathbf{x}:\mathbf{v}^{\top}\mathbf{x}=0\} =0$. Here $c:=\mathbb{E}_{\mathbf{x}\sim\mathcal{D}}\left\|\mathbf{x}\right\|$.
\end{prop}

Generally speaking, gradient methods have no linear convergence for smooth, Lipschitz continuous, and convex but not strongly convex functions \cite{nesterov2013introductory}. 
However, some invariant properties (such as Proposition \ref{prop:grad-norm-prop} below) make the training dynamic reserve linear directional convergence.

\begin{prop}\label{prop:grad-norm-prop}
	Suppose $P\left(\{ \mathbf{x}:\mathbf{v}^{\top}\mathbf{x}=0\}\right)<1$. Then we have that  $\mathbf{v}^\top\nabla L(\mathbf{w})< 0$ for any $\mathbf{w}\in\mathbb{R}^d$.	
	Therefore, $\mathbf{v}^\top\mathbf{w}(t)$ is increasing, and $\|\mathbf{w}(t)\|$ is unbounded for gradient flow. If the learning rates are lower bounded, i.e. $\eta_n\geq \eta_->0$, then $\mathbf{v}^\top\mathbf{w}(n)$ is increasing, and $\|\mathbf{w}(n)\|$ is unbounded.
\end{prop}
As mentioned in previous work, it is  a common phenomenon that the  norm of the weight vector is infinite when data are perfectly classified. And we further verify it even under the population loss and zero margin dataset.

\subsection{Spherically Symmetric Data Distributions}
Now we give convergence rate under gradient methods. In the following analysis, we make the assumption that $\mathcal{D}$ is spherically symmetric used in \citet{yehudai2020learning}, which includes the standard Gaussian as a special case. Using this assumption, we are able to reduce the scope of optimization into the plane $\text{span}\{\mathbf{w}(0), \mathbf{v}\}$. 

\begin{assume}\label{ass:distri}
	Assume $\mathbf{x}\sim\mathcal{D}$ has a spherically symmetric distribution, i.e., for any orthogonal matrix $A: A \mathbf{x} \sim\mathcal{D}$.
\end{assume}

\begin{prop}\label{prop:grad-prop}
	Under Assumption \ref{ass:distri}, we have 	$\|\nabla L(\mathbf{w})\| \leq c_0$ and $ \nabla L(r\mathbf{v}) = -\mathbb{E}_{\mathbf{x}\sim\mathcal{D}} \frac{|x_1|}{1+e^{r|x_1|}}\mathbf{v}, \forall \ r \in \mathbb{R}$.
\end{prop}

\begin{remark}
	Proposition \ref{prop:grad-prop} shows that if $\mathbf{w}$ aligns with $ \mathbf{v}$ (including $\mathbf{0}$), then gradient methods always direct to $\mathbf{w} = r\mathbf{v}$ with $r\to +\infty$ according to Proposition \ref{prop:grad-norm-prop}. Therefore, we only need to consider the initial value $\mathbf{w}(0) \neq\mathbf{0}$ and $\theta(0)\neq 0$ or $\pi$. Moreover, $c_0$ does not depend on dimension $d$ under Assumption \ref{ass:distri}.
\end{remark}

The variation of $\theta(\mathbf{w}, \mathbf{v})$ is the most we concern. In the gradient flow setting, it follows that for $\mathbf{w}(t)\neq \mathbf{0}$,
\[ \frac{\partial \cos\theta(t)}{\partial t} = {-}\frac{1}{\|\mathbf{w}(t)\|}\left(\mathbf{v}{-}\left(\overline{\mathbf{w}}(t)^{\top} \mathbf{v}\right) \overline{\mathbf{w}}(t)\right)^{\top} \nabla L(\mathbf{w}(t)). 
\]

Our analysis focuses on $\|\mathbf{w}(t)\|$ and the remaining part. Now we present the lemma below to show the exact directional improvement:
\begin{lemma} \label{lemma:angle_grad} 
	Under Assumption \ref{ass:distri} and if $\mathbf{w} \neq \mathbf{0}$, then
	\[ -\left(\mathbf{v}-\left(\overline{\mathbf{w}}^{\top} \mathbf{v}\right) \overline{\mathbf{w}}\right)^{\top} \nabla L(\mathbf{w}) = \frac{c_0\sin^2\theta(\mathbf{w}, \mathbf{v})}{\pi}. \]
	Moreover, invoking from the above result, we have the following corollary: 
	\[ -\left(I-\overline{\mathbf{w}} \; \overline{\mathbf{w}}^{\top}\right)\nabla L(\mathbf{w}) = \frac{c_0}{\pi} \left(I-\overline{\mathbf{w}} \; \overline{\mathbf{w}}^{\top}\right)\mathbf{v}. \]
\end{lemma}


\begin{figure}[t]
	\centering
	\includegraphics[width=0.3\textwidth]{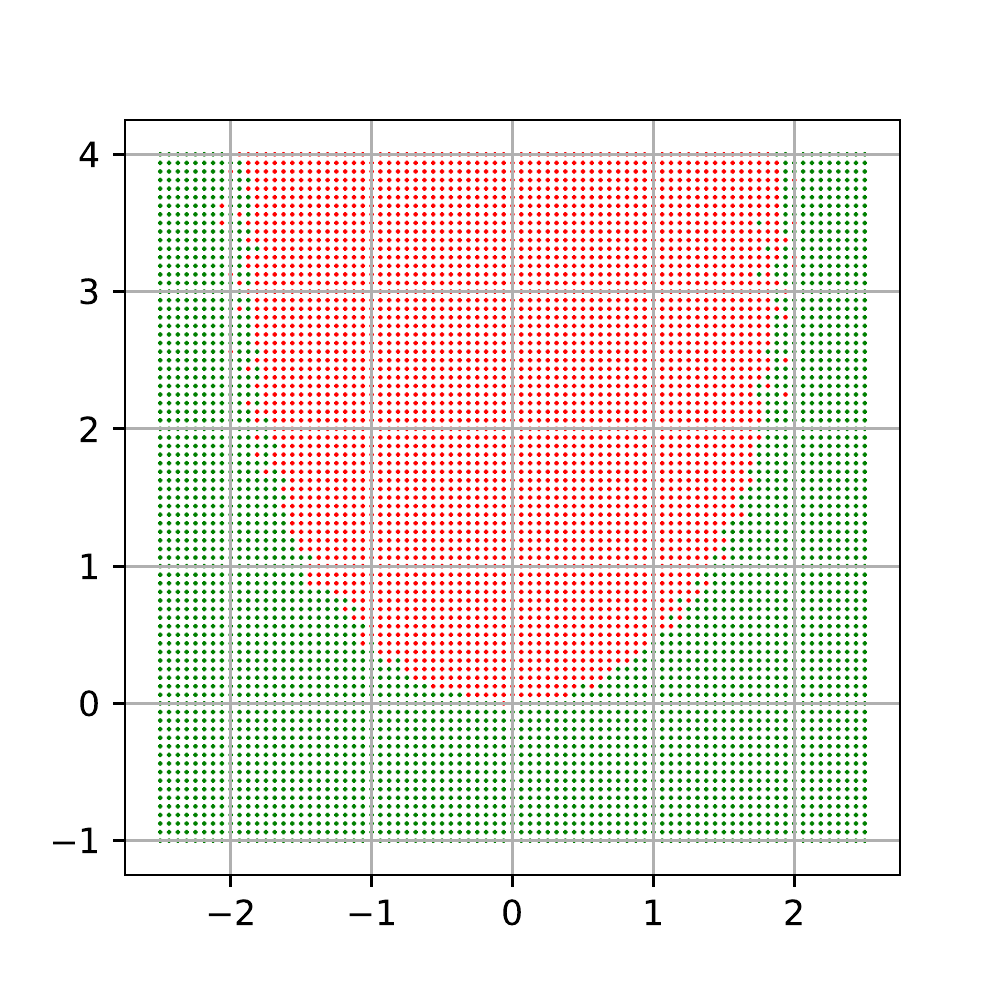}
	\caption{Verifying $\text{sgn}(N(\mathbf{w}))$ when $\mathbf{x}\sim\mathcal{U}(\mathcal{S}^1)$ and $\mathbf{v}=(0,1)^\top$. The sign of $N(\mathbf{w})$ at each $\mathbf{w}$ is estimated by $1000$ random samples. Red means positive sign while green is negative sign.} 
	\label{fig:norm-change}
\end{figure}

\subsection{Gradient Flow}
From Lemma \ref{lemma:angle_grad}, we obtain that $\frac{\partial \cos\theta(t)}{\partial t} \geq 0$ when $\mathbf{w}(t) \neq 0$.
Therefore,  to derive directional convergence, we need to understand the variation of $\|\mathbf{w}(t)\|$ during optimization dynamic. We set $N(\mathbf{w}(t)):=-\mathbf{w}(t)^{\top} \nabla L(\mathbf{w}(t))=\frac{\partial \|\mathbf{w}(t)\|^2}{\partial t}$. Then we have the following propositions as  depicted in Figure \ref{fig:norm-change}.

\begin{lemma}\label{lemma:norm-change}
	Under Assumption \ref{ass:distri}, the following propositions hold:
	1) $N(\mathbf{w})\leq 0.3$;
	2) If $\theta(\mathbf{w}, \mathbf{v})\geq \pi/2$, then $N(\mathbf{w})\leq 0$;
	3) If $\theta(\mathbf{w}, \mathbf{v})\leq \pi/2$ and $\|\mathbf{w}\|$ is fixed, then $N(\mathbf{w})$ increases when $\theta(\mathbf{w}, \mathbf{v})$ decreases. 	
	More generally, if $r \leq \|\mathbf{w}\| \leq R$ (but $\|\mathbf{w}\|$ is not fixed), then there exists $\theta_0$ (related to $r,R$) such that $N(\mathbf{w})>0$ when $\theta(\mathbf{w}, \mathbf{v})<\theta_0$;
	4) If $0 < \|\mathbf{w}\|\leq \frac{2|\cos\theta(\mathbf{w}, \mathbf{v})|}{\pi}$, then $N(\mathbf{w})\cos\theta(\mathbf{w}, \mathbf{v})>0$. 
\end{lemma}

\begin{lemma} \label{lemma:norm-increase-linear}
	Under Assumption \ref{ass:distri} and if $\partial \|\mathbf{w}(t)\|^{2}/\partial t=0$ for some $\mathbf{w}(t) \neq \mathbf{0}$, then $\partial \cos\theta(\mathbf{w}(t), \mathbf{v})/\partial t \geq 0$.
\end{lemma}
Therefore, based on Proposition \ref{prop:grad-norm-prop}, Lemmas \ref{lemma:norm-change} and \ref{lemma:norm-increase-linear}, $\|\mathbf{w}(t)\|$ first decreases (if $\mathbf{w}(0)^{\top}\nabla L(\mathbf{w}(0)) < 0$) then increases to $+\infty$. Therefore, the directional convergence includes the following two phases.  

\begin{thm}\label{thm:linear-flow}
	Under Assumption \ref{ass:distri}, we obtain the following two-phase directional convergence for $\mathbf{w}(0) \neq \mathbf{0}$ and $\theta(0)\neq \pi$. If $N(\mathbf{w}(0)) < 0$, then there exists some $T>0$ such that $N(\mathbf{w}(T)) = 0$, otherwise we set $T=0$. With such a $T$, we have that 
	\[ \cos\theta(t)\geq \left\{ \begin{array}{ll}
	 1-\frac{2}{e^{A_1t+B_1}+1},              & t \leq T, \\
	1-\frac{2}{e^{A_2\sqrt{t-T+C_2}+B_2}+1},  & t \geq T.
	\end{array}  \right.
	\]
	Here $A_1 {=} \frac{2c_0}{\pi\|\mathbf{w}(0)\|}$, $\ B_1 {=} -2\ln\left|\tan\frac{\theta(0)}{2}\right|$, 
	$A_2=\frac{4c_0}{\sqrt{0.6}\pi}$, $B_2{=} -2\ln\left|\tan\frac{\theta(T)}{2}\right|-\frac{4c_0\|\mathbf{w}(T)\|}{0.6\pi}$, and $C_2=\frac{\|\mathbf{w}(T)\|^2}{0.6}$.
\end{thm}

We find that $\|\mathbf{w}\|$ increases much slower than $\theta$ decreases, showing that the objective is nearly strongly convex, which provides linear (or logarithmic) convergence. The time complexity is $\mathcal{O}\left(\ln^2(1/\epsilon)\right)$ to ensure $1-\cos \theta(t) \leq \epsilon$. 
Moreover, our convergence bound gives faster directional convergence compared to \citet{nacson2019stochastic, soudry2018implicit}. The main difference comes from the structured data in Assumption \ref{ass:distri}, showing the benefit from data augmentation and preprocessing.
In addition, adding the regularization term $\lambda\|\mathbf{w}\|^2$ would accelerate the convergence, because $\|\mathbf{w}\|$ increases much slow and the loss function is already strongly convex, though we focus the original problem rather than the regularization.

\subsection{Gradient Descent}
Now we turn to the gradient descent setting based on the previous results. The difficulty we encounter is that the choice of the learning rate  may break the directional monotonicity in Lemma \ref{lemma:angle_grad}. Fortunately, the first phase directional convergence in the gradient flow still holds.

\begin{thm}\label{thm:negative-angle}
	Under Assumption \ref{ass:distri}, if $\theta(0)\neq \pi$, $\mathbf{v}^{\top}\mathbf{w}(0) < 0$, and setting $S_n^- :=\sum_{k=0}^{n-1} \frac{\eta_k}{\sqrt{A+\sum_{i=0}^{k}\eta_i^2}}$, we have that
	\[ \cos\theta(n)\geq 1-\left(1-\cos\theta(0)\right) e^{-BS_n^-}, \]
	until $\cos\theta(n)\geq 0$. Here $A=\frac{\|\mathbf{w}(0)\|^2}{c_0^2}$ and $B = \frac{1+\cos\theta(0)}{\pi}$.
\end{thm}

\begin{remark}
	Obviously, $S_n^-<n$. And we list several choices of $\{\eta_n\}_{i=1}^\infty$:
	\[ S_n^-=\left\{
	\begin{aligned}
	&\Theta(n), & & \eta_n=\Theta(q^n), q>1; \\
	&\Theta(\sqrt{n}), & & \eta_n=\Theta(n^{\alpha}), \alpha > -1/2; \\
	&\Omega(\sqrt{n/\ln(n)}), & & \eta_n=\Theta(n^{\alpha}), \alpha = -1/2; \\
	&\Theta(n^{\alpha+1}), & & \eta_n=\Theta(n^{\alpha}), -1 < \alpha < -1/2; \\
	&\Theta(\ln(n)), & & \eta_n=\Theta(n^{\alpha}), \alpha = -1; \\
	&<\infty, & & \eta_n=\Theta(n^{\alpha}), \alpha < -1. 
	\end{aligned}
	\right.\]
\end{remark}

Hence, when weight $\mathbf{w}(n) $ stays in the `wrong' region with $\theta(n) > \pi/2$, larger learning rate gives faster directional convergence to the region $\{\mathbf{w}: \theta(\mathbf{w},\mathbf{v})\geq 0\}$. Unfortunately, when $\theta(n)\leq \pi/2$, the phase becomes unstable and heavily depends on the learning rate. To inherit the directional monotonicity, we need a sufficient condition in the whole training period. Moreover, we will show that this condition can be satisfied when the weight norm is large enough compared to the learning rate.  

\begin{thm}[A sufficient convergence condition] \label{thm:suff}
	Under Assumption \ref{ass:distri}, if there exists a $\delta>0$ such that $\forall \; n\in\mathbb{N}$, 
	\begin{equation}\label{eq:suff}
	\|\mathbf{w}(n{+}1)\|+\overline{\mathbf{w}}(n)^{\top} \mathbf{w}(n{+}1) \geq \frac{(1 {+} \delta)c_0\eta_n \cos\theta(n)}{\pi},
	\end{equation}
	then there exist constants $A, B, C>0$ such that 
	\[ \cos\theta(n) \geq  1-\left(1-\cos\theta(0)\right) e^{-BS_n^+}. \]
	Here $S_n^+ := \sum_{k=0}^{n-1} \frac{\eta_k}{\sqrt{A+\sum_{i=0}^{k}\eta_i^2+C\eta_i}}$, $ A=\frac{\|\mathbf{w}(0)\|^2}{c_0^2}$, $B = \frac{\delta(1+\cos\theta(0))}{\pi+\delta\pi}$, and $C=\frac{0.6}{c_0^2}$.
\end{thm}

\begin{remark}
	When $\|\mathbf{w}(n)\|\geq \eta_n c_0+(1+\delta)c_0\eta_n /(2\pi)$,  Eq.~(\ref{eq:suff}) can be satisfied because
	\begin{equation*}
	\begin{aligned}
	& \|\mathbf{w}(n+1)\|+ \overline{\mathbf{w}}(n)^{\top} \mathbf{w}(n+1) \\
	& \geq 2\left( \|\mathbf{w}(n)\|-\eta_n\|\nabla L(\mathbf{w}(n))\|\right)\\
	&  \geq (1+\delta)c_0\eta_n \cos\theta(n)/\pi.
	\end{aligned}
	\end{equation*}
	Once $\eta_n \leq \eta_+$, then $\|\mathbf{w}(n)\|\geq R_1:=\eta_+ c_0 +c_0\eta_+ /\pi$ is enough to derive the sufficient convergence condition.
\end{remark}

\begin{remark}
	Obviously, $S_n^+<n$. And we list several choice of $\{\eta_n\}_{i=1}^\infty$:
	\[ S_n^+=\left\{ \begin{aligned}
	&\Theta(n), & & \eta_n=\Theta(q^n), q>1; \\
	&\Theta(\sqrt{n}), & & \eta_n=\Theta(n^{\alpha}), \alpha \geq 0; \\
	&\Theta(n^{(\alpha+1)/2}), & & \eta_n=\Theta(n^{\alpha}), -1 < \alpha < 0; \\
	&\Theta(\sqrt{\ln(n)}), & & \eta_n=\Theta(n^{\alpha}), \alpha = -1; \\
	&<\infty, & & \eta_n=\Theta(n^{\alpha}), \alpha < -1. 
	\end{aligned}
	\right. \]
\end{remark}

To further derive more realistic result out of the sufficient convergence condition, we show directional convergence with bounded learning rates below. The idea is to combine the increasing property of $\mathbf{w}^\top\mathbf{v}$,  giving enough directional monotonicity update steps. Similar results have shown in more general settings under benign initialization and positive margin dataset \cite{ji2020directional, lyu2019gradient, nacson2019stochastic}.
 
\begin{thm}\label{thm:angle-convergence}
	Under Assumption \ref{ass:distri}, for an arbitrary choice of learning rate sequence $\{\eta_n\}_{i=1}^{\infty}$ with $\eta_{+} \geq \eta_n \geq \eta_{-}>0$, there exists a subsequence $\{n_k\}_{k=1}^\infty$, which guarantees the linear directional convergence of $\cos\theta(n_k)$ to $1$. 
\end{thm}

\begin{remark}
	In general, $\beta$-smooth and convex objective has the learning rate constraint ($\eta\leq \frac{2}{\beta}$) to guarantee the convergence. But in the current concern,  there is no need for this constraint because the purpose is learning the direction instead of decreasing the alternative loss function. 
\end{remark}

From Theorem \ref{thm:angle-convergence}, we always have directional convergence for a bounded learning rate sequence. Thus, motivated by the infinite weight norm and increasing projection on $\mathbf{v}$ by Proposition \ref{prop:grad-norm-prop}, we always reach some time that  $\|\mathbf{w}(n)\|$ is large enough, which gives a sufficient convergence condition and leads to logarithmic directional convergence.

\begin{corollary}\label{corr:angle-convergence}
	Under the assumptions in Theorem \ref{thm:angle-convergence}, there exists $n_0 > 0$ such that $\|\mathbf{w}(n_0)\|\cos\theta(n_0)\geq \eta_+c_0+\eta_+c_0/\pi$, and gradient descent will give logarithmic convergence of $\cos\theta(n)$ to $1$ from $n_0$. 
\end{corollary}
Comparing to the previous fixed learning rate based convergence results for linear predictor \cite{soudry2018implicit, nacson2019convergence, nacson2019stochastic}, we derive directional convergence under more fixable bounded learning rates.
Since our concerned problem has inductive bias in origin, we also show two initialization methods to guarantee the convergence from beginning in Appendix \ref{app:init}.

\section{Deep Linear Networks}\label{sec:deep-linear}
Invoking from the linear predictor and previous work on deep linear networks, we extend the results to deep linear networks and leave details in Appendix \ref{app:deep-linear}. For a $N$-layer linear network $\phi(\mathbf{x}, \mathbf{w})=W_N\dots W_1\mathbf{x}$ where $\mathbf{w}=\left(W_N, \ldots, W_1\right)$, the objective becomes
\begin{equation*}
\min_{\mathbf{w}} L^N(W_N,\dots,W_1) :={\mathbb E}_{\mathbf{x}\sim\mathcal{D}}\ln\left[1{+}e^{-y(\mathbf{x})W_N\cdots W_1\mathbf{x}}\right].
\end{equation*} 
Every such a network represents a linear mapping given as $\mathbf{w}_e = \left(W_N\cdots W_1\right)^\top\in \mathbb{R}^d$:
\[ L^N(W_1, \dots, W_N){=} L^1(\mathbf{w}_e){=}\mathbb{E}_{\mathbf{x}\sim\mathcal{D}} \ln\left(1{+}e^{-y(\mathbf{x})\mathbf{w}_e^\top\mathbf{x}}\right). \]
A key tool for analyzing the induced flow for $\mathbf{w}_e$ is established
in Claim 2 of \citet{arora2018optimization}. If the initial balancedness conditions: 
\[ W_{j+1}(0)^\top W_{j+1}(0) = W_j(0) W_j(0)^\top, j=1,\dots, N-1, \]
hold, then we have that
\[ \frac{\partial \mathbf{w}_e}{\partial t} {=} -\|\mathbf{w}_e\|^{2-\frac{2}{N}} \left(\frac{d L^1(\mathbf{w}_e)}{d\mathbf{w}}{+}(N{-}1)\overline{\mathbf{w}}_e\overline{\mathbf{w}}_e^\top \frac{d L^1(\mathbf{w}_e)}{d\mathbf{w}} \right). \]
Similarly, we can build up the monotonic directional improvement in the following lemma.

\begin{lemma} \label{lemma:deep_linear_grad} 
	Under Assumption \ref{ass:distri} and initial balancedness condition, if $\mathbf{w}_e(t) \neq \mathbf{0}$, then
	\[ \frac{\partial \cos\theta(\mathbf{w}_e(t), \mathbf{v})}{\partial t} = \frac{c_0\sin^2\theta(\mathbf{w}_e(t), \mathbf{v})}{\pi}\|\mathbf{w}_e(t)\|^{1-\frac{2}{N}}.\]
\end{lemma}
The main difference from the shallow linear network is that the dependence of $\|\mathbf{w}_e\|$ is reversed. Large $\|\mathbf{w}_e\|$ gives faster convergence for the deep linear network when $N \geq 3$, but for $N=1$ it is clearly opposite. Thanks to the similar expression of the gradient of the induced weight norm, we still have at least two-phase directional convergence. The only potential difficulty is that $\mathbf{w}_e(t)$ may converge to the potential stationary point at the origin (at which the angle is not well-defined). Fortunately, this cannot happen by Lemma \ref{lemma:deep-linear-norm}.

\begin{lemma}\label{lemma:deep-linear-norm}
	Under Assumption \ref{ass:distri} and initial balancedness condition, if $\mathbf{w}_e(0) \neq \mathbf{0}$, and $N > 2$, then
	\[ \left(\|\mathbf{w}_e(0)\|^{\frac{2}{N}} + 0.6t\right)^{\frac{N}{2}}\geq \|\mathbf{w}_e(t)\|, \] 
	\[ \|\mathbf{w}_e(t)\|\geq \left(\|\mathbf{w}_e(0)\|^{\frac{2}{N}-1} + (N-2)c_0t\right)^{-\frac{N}{N-2}}>0. \]
\end{lemma}

\begin{thm}\label{thm:deep-linear-convergence1}
	Under Assumption \ref{ass:distri} and initial balancedness condition, if $N>2$, $\mathbf{w}_e(0) \neq \mathbf{0}$ and $\theta(0)\neq \pi$, then we obtain two-phase convergence as follows. 
	If $\partial \|\mathbf{w}_e(0)\|^2/\partial t < 0$, then there exists $T>0$ such that $\partial \|\mathbf{w}_e(T)\|^2/\partial t = 0$, otherwise, we set $T=0$. With such a $T$, it holds that 
	\[ \cos\theta(t)\geq \left\{ \begin{array}{ll}
	   1-\frac{2}{C_1(A_1t/B_1+1)^\alpha+1},    & t\leq T, \\
	  1-\frac{2}{1+e^{A_2(t-T)+B_2}},  &  t\geq T. 
	\end{array}  \right.
	\]
	Here $A_1{=}(N{-}2)c_0$, $B_1{=}\|\mathbf{w}_e(0)\|^{\frac{2}{N}-1}$, $C_1{=}\frac{1+\cos\theta(0)}{1-\cos\theta(0)}$, $A_2{=}\frac{2c_0\|\mathbf{w}_e(T)\|^{2-\frac{2}{N}}}{\pi}$, $B_2{=}-2\ln \left|\tan\frac{\theta(T)}{2}\right|$, $\alpha=2c_0/\pi$.
	
	In addition, we have the upper bound that 
	\[ \cos\theta(t) \leq 1-\dfrac{2}{e^{F[\left(0.6t+D\right)^{N/2}-D^{N/2}]+E}+1}, \]
	where $D{=}\|\mathbf{w}_e(0)\|^{\frac{2}{N}}$, $E{=}-2\ln \left|\frac{\tan\theta(0)}{2}\right|$, $F{}=\frac{4c_0}{0.6N\pi}$.
\end{thm}

\begin{remark}
	As $N$ increases,  $A_1=(N-2)c_0$ also increases. Thus the lower bound of $\cos\theta(t)$ increases, and $\cos\theta(t)$ converges faster, which is consistent with the implicit acceleration of large layers in \citet{arora2018optimization}. 
	
	In addition, we can see that $\|\mathbf{w}_e(t)\|$ behaves similarly with the case $N=1$, while it first decreases (if possible) then increases, showing that if we start with $\mathbf{w}_e(0) \neq k\mathbf{v}, k\leq 0$ we would never converge to the origin from the proof in Theorem \ref{thm:deep-linear-convergence1}. When $\theta(0) = \pi$, then $\forall t\geq 0, \ \theta(t) = \pi$, and $\mathbf{w}_e(t) \to \mathbf{0}$ but never hits the origin.
\end{remark}

\begin{remark}
	Note that $\theta(t) \to 1$ except  $\mathbf{w}_e(0) =k\mathbf{v}$ for $k\leq 0$. We can also guarantee $\mathbf{w}_e(t) \to \infty$ because $\mathbf{v}^\top\mathbf{w}_e(t)$ is increasing after some time (but not always). Hence, a larger norm of $\mathbf{w}_e(t)$ gives faster convergence rate. However, larger norm leads to much slow increasing of the weight norm, particularly in a negative exponential rate, which would result in $\|\mathbf{w}_e(t)\|=\Theta(\ln(t))$ in \citet{soudry2018implicit, nacson2019convergence}. We would treat this as the third phase, but roughly the direction of $\mathbf{w}_e(t)$ already approaches to the target. 
\end{remark}

\section{Two-layer Non-linear Network with Two Hidden Nodes} \label{sec:shallow-non-linear}
Despite of directional understanding of (deep) linear networks, we conceptually have more difficulties for nonlinear networks. 
A popular line of the recent developments shows how gradient methods on highly over-parameterized neural networks can learn various target functions such as cross entropy loss in polynomial time. Here we  discuss a simple,  under-parameterized setting. 
Using one hidden neuron with nonlinear activation as $\sigma(\cdot)$ would give constraints in the classification probability and may not guarantee the recovery of $\mathbf{v}$. We go further to see the case of two hidden neurons. Moreover, we fix the second layer with different signs to describe the optimization dynamic of the first layer weight. Notice that the homogeneity may be broken if we do not use homogeneous activation. Now the classifier becomes
$\phi(\mathbf{x}) = \sigma(\mathbf{w}^{\top}_1\mathbf{x})-\sigma(\mathbf{w}^{\top}_2\mathbf{x})$.
Then the objective becomes:
$\min_{\mathbf{w}_1, \mathbf{w}_2} L(\mathbf{w}_1, \mathbf{w}_2) :=  \mathbb{E}_{\mathbf{x}\sim\mathcal{D}}\ln\left(1+e^{-y(\mathbf{x})\phi(\mathbf{x})}\right)$.
So we make the following assumption for activation to guarantee the recovery of target.
\begin{assume}\label{assume:activation} The activation 
	$\sigma\colon \mathbb{R}\to\mathbb{R}$ is monotonically non-decreasing and satisfies $\inf_{0 \leq z\leq M} \sigma'(z) \geq \gamma(M)>0, \ \forall M > 0, a.e.$ (Note by the Lebesgue Differentiation Theorem, we can define $\sigma'(z), a.e.$).
\end{assume}
Assumption \ref{assume:activation} for the activation function covers most activations used in practice such as ReLU and standard sigmoidal activations (for which the derivative in any bounded interval is lower bounded by a positive constant). 
We denote $\nabla_{i} L(\mathbf{w}) {:=} \nabla_{\mathbf{w}_i} L(\mathbf{w}_1, \mathbf{w}_2)$ for $i=1, 2$.
Similarly, we consider gradient flow and gradient descent methods. For simplifying the proof, we make initialization constraint assumption as follows:
%
\begin{assume}\label{assume:same_plane}
	Assume that $\mathbf{w}_1(0)$, $\mathbf{w}_2(0)$, and $\mathbf{v}$ are in the same plane, i.e., $\mathbf{v} \in \mathrm{span}\{\mathbf{w}_1(0), \mathbf{w}_2(0)\}$.
\end{assume}

Assumption \ref{assume:same_plane} for the initial weights guarantees $\mathbf{w}_i(t)$ (or $\mathbf{w}_i(n)$) staying in the same plane decided by initial weights because $\nabla_i L(\mathbf{w})\in \text{span}\{\mathbf{w}_1, \mathbf{w}_2, \mathbf{v}\}$ when Assumption \ref{ass:distri} holds. Thus we have well-defined gradient and iterative methods. We denote $\theta_i=\theta(\mathbf{w}_i, \mathbf{v})$, $\theta_i(t)=\theta(\mathbf{w}_i(t), \mathbf{v})$, $\theta_i(n)=\theta(\mathbf{w}_i(n), \mathbf{v}), \ i=1,2$, as described earlier. Before deriving our result, we need to have some intuitions shared with the linear case in the previous sections.

\begin{prop}\label{prop:grad-update}
	Under Assumption \ref{assume:activation} and if $P\left(\mathbf{x}=\mathbf{0}\right)<1$, then $\mathbf{v}^\top\nabla_1 L(\mathbf{w})<0$ and $\mathbf{v}^\top\nabla_{2} L(\mathbf{w})>0$.
	Hence $\mathbf{v}^\top\mathbf{w}_1(t)$ is increasing, $\mathbf{v}^\top\mathbf{w}_2(t)$ is decreasing, and $\|\mathbf{w}_1(t)\|$ or $\|\mathbf{w}_2(t)\|$ is unbounded. Similarly, $\mathbf{v}^\top\mathbf{w}_1(n)$ is increasing, $\mathbf{v}^\top\mathbf{w}_2(n)$ is decreasing, and $\|\mathbf{w}_1(n)\|$ or $\|\mathbf{w}_2(n)\|$ is unbounded when applied with lower bounded learning rate sequence $\eta_n\geq\eta_->0$.
\end{prop}

In addition, we need to emphasize that the recovery happens only when $\mathbf{w}_1{-}\mathbf{w}_2=k\mathbf{v}$ for $k>0$, from  monotonicity of the activation. 
However, when using  partial zero value activation (such as ReLU), we need more rigorous analysis of the recovery. 
At the first glance, we would choose $\mathbf{w}_1{-}\mathbf{w}_2$ as the objective, but we discover complex process from the experiments in Section \ref{sec:exp}. 
As a second thought, we find the monotonicity of $\theta_1(t)$ and $\theta_2(t)$ in some scenarios as well.
We make use of Assumptions \ref{ass:distri}, \ref{assume:activation} and \ref{assume:same_plane}, proving the following  key technical lemma, which implies that weights towards the `correct' direction for at least half the possible position of $\mathbf{w}_1, \mathbf{w}_2$ and $\mathbf{v}$. We defer the proofs to Appendix \ref{app:shallow-nonlinear}.

\begin{lemma} \label{lemma:ac-angle-grad}
	Under Assumptions \ref{ass:distri}, \ref{assume:activation} and \ref{assume:same_plane}, suppose $\mathbf{w}_1 \neq \mathbf{0}$ and $\left(\mathbf{v}-(\overline{\mathbf{w}}_1^\top\mathbf{v})\overline{\mathbf{w}}_1\right)^\top\left(\mathbf{w}_2-(\overline{\mathbf{w}}_1^\top\mathbf{w}_2)\overline{\mathbf{w}}_1\right)\geq 0$, which means that $\mathbf{w}_2$ and $\mathbf{v}$ are in the same half-plane separated by $\mathbf{w}_1$. Then we have that
	\[ {-} \left(\mathbf{v} {-} (\overline{\mathbf{w}}_1^\top\mathbf{v})\overline{\mathbf{w}}_1\right)^\top\nabla_{1} L(\mathbf{w}) \geq \frac{\nu(\|\mathbf{w}_1\sin\theta_1\|)}{2\pi} \sin^2 \theta_1, \]
	where $\nu(z) := \mathbb{E}_{\mathbf{x} \sim\mathcal{D}_2} \left[\sigma(z\|\mathbf{x}\|)-\sigma(-z\|\mathbf{x}\|) \right] / z ,\; z > 0$.
	Similarly, if $\left(\mathbf{v} {-} (\overline{\mathbf{w}}_2^\top\mathbf{v}) \overline{\mathbf{w}}_2\right)^\top\left(\mathbf{w}_1 {-} (\overline{\mathbf{w}}_2^\top\mathbf{w}_1) \overline{\mathbf{w}}_2\right)\leq 0$ and $\mathbf{w}_2 \neq \mathbf{0}$, which means that $\mathbf{w}_1$ and $\mathbf{v}$ are in the different half-plane separated by $\mathbf{w}_2$, then
	\[ \left(\mathbf{v}-(\overline{\mathbf{w}}_2^\top\mathbf{v})\overline{\mathbf{w}}_2\right)^\top\nabla_{2} L(\mathbf{w}) \geq \frac{\nu(\|\mathbf{w}_2\sin\theta_2\|)}{2\pi} \sin^2 \theta_2. \]
\end{lemma}

Lemma \ref{lemma:ac-angle-grad} introduces the positive (or negative) increment of angle between the weight and target. We note that the optimization analysis for general activation may have exponentially small (such as sigmoid, tanh) tail, yielding slow directional convergence rate. 
In addition, if $\sigma(z)$ is bounded (this time it doesn't satisfy statistic scope of logistic regression) and strictly monotonically increasing, the recovery happens when $\mathbf{w}_1{-}\mathbf{w}_2$ aligns with positive direction of $\mathbf{v}$. 
An interesting question for future work is to obtain the directional convergence of $\mathbf{w}_1{-}\mathbf{w}_2$ under bounded and increasing activation. Here we focus on unbounded activation and totally alignment for recovery.

\subsection{Convergence for ReLU Activation}

In this section we consider the standard ReLU function which is broadly employed in many neural networks. In more details, we impose the convention that even though the ReLU function is not
differentiable at $ 0 $, we take $\sigma(0)$ as some fixed positive number to meet Assumption \ref{assume:activation}. With such a ReLU function, we are able to provide logarithmic directional convergence as the linear case with strengthened property of objective and optimization dynamic. 
  

\begin{prop}\label{prop:unbounded}
	Under Assumptions \ref{ass:distri}, \ref{assume:activation} and \ref{assume:same_plane}, and $\sigma(z)=\max\{0,z\}$, then $\|\mathbf{w}_1(t)\|$ and $\|\mathbf{w}_2(t)\|$ are unbounded. And $\|\mathbf{w}_1(n)\|$ and $\|\mathbf{w}_2(n)\|$ are unbounded when applied with lower bounded learning rate sequence $\eta_n\geq\eta_->0$.
\end{prop}

\begin{lemma}\label{lemma:relu-norm}
	$\partial \left(\|\mathbf{w}_1(t)\|^2+\|\mathbf{w}_2(t)\|^2\right)/\partial t \leq 0.6$.
\end{lemma}
\vspace{-8pt}
\paragraph{Gradient Flow.}  Now we give directional convergence invoked from Lemmas \ref{lemma:ac-angle-grad} and \ref{lemma:relu-norm} when $\mathbf{w}_1(0)$ and $ \mathbf{w}_2(0)$ are initialized in the different half-plane separated by $\mathbf{v}$.
\begin{thm} \label{thm:diff-init}
	Under Assumptions \ref{ass:distri}, \ref{assume:activation} and \ref{assume:same_plane}, if $\mathbf{w}_1(0)$ and $\mathbf{w}_2(0)$ lie in the different half-plane separated by $\mathbf{v}$, then for any $t>0$, there exists $i\in\{1,2\}$, such that
	\[ (-1)^{i-1}\cos\theta_i(t)\geq 1-\frac{2}{e^{A\sqrt{t+B}+C}+1}. \]
	Here $A{=}\frac{2c_0}{\pi\sqrt{0.6}}, B{=}\frac{r(0)^2}{0.6}$, $r(0)^2{=}\|\mathbf{w}_1(0)\|^2{+}\|\mathbf{w}_2(0)\|^2$, and $C=-\frac{2c_0r(0)}{0.6\pi}+2\ln\tan\frac{\max\{\theta_1(0),\pi-\theta_2(0)\}}{2}$. 
\end{thm}

\begin{remark}
	Theorem \ref{thm:diff-init} gives the convergence guarantee at least for one weight in $\mathbf{w}_1(t)$ and $\mathbf{w}_2(t)$. If we know the induction bias introduced by the learning task,  it is enough to finish the learning task. Informally, we show the complex process of the  case in which one weight already converges but another not. That is,  
\end{remark}

\begin{thm} \label{thm:diff-init2}
	Under Assumptions \ref{ass:distri}, \ref{assume:activation} and \ref{assume:same_plane}, if $\theta_1(0) = 0$, then for some $t_0\geq0$, $\mathbf{w}_2(t_0)^\top\mathbf{v}\leq 0$;  and if $\mathbf{w}_2(t_0)\neq \mathbf{0}$, then for $t>t_0$, $1+\cos\theta_2(t)$ may go through convergence rate between $-\Theta(\ln t) \sim e^{-\Theta(\sqrt{t})}$.
\end{thm}

For the remaining initialization case that $\mathbf{w}_1(0)$ and $\mathbf{w}_2(0)$ are initialized in the same half-plane separated by $\mathbf{v}$, we only discover that $\theta_1(t) \leq \theta_2(t)$ always holds for some $t$ as follows, and the remaining dynamic of directional convergence seems complex and does not have distinct phases, which we will show experiments in Section \ref{sec:exp}.

\begin{thm}\label{thm:same-init}
	Under Assumptions \ref{ass:distri}, \ref{assume:activation} and \ref{assume:same_plane}, if $\mathbf{w}_1(0)$ and $\mathbf{w}_2(0)$ lie in the same half-plane separated by $\mathbf{v}$, and $\theta_1(0) > \theta_2(0)$, then $\theta_1(t)$ decreases and $\theta_2(t)$ increases. Moreover, $\theta_1(t) \leq \theta_2(t)$ after time $\mathcal{O}\left(\ln^2\delta \right)$, where $\delta = \cos\theta_2(0)-\cos\theta_1(0)$.
\end{thm}

\paragraph{Gradient Descent.} We close our theoretical discussion by extending the gradient flow result to the gradient descent method. On the account of the unsolved initialization case in Theorem \ref{thm:same-init}, we are only able to give directional convergence for constrained learning rates. The main idea follows from the linear setting with more rigorous analysis. 

\begin{thm}\label{thm:gd-relu}
	Suppose that Assumptions \ref{ass:distri}, \ref{assume:activation} and \ref{assume:same_plane} hold and that $\mathbf{w}_1(0)$ and $\mathbf{w}_2(0)$ lie in the different half-plane separated by $\mathbf{v}$. Moreover, $\{\eta_n\}_{i=1}^\infty$ satisfy that $\mathbf{w}_i(n)$ would never reach another half-plane separated by $\mathbf{v}$. Then for each step $n$, there exists $i\in\{1,2\}$, such that
	\[ (-1)^{i-1}\cos\theta_i(n) \geq 1-(1+(-1)^i\cos\theta_i(0))e^{-BS_n}, \]
	where $ S_n = \sum_{k=0}^{n-1} \frac{\eta_k}{\sqrt{\|\mathbf{w}_1(0)\|^2 + \|\mathbf{w}_2(0)\|^2+\sum_{i=0}^{k}2\eta_i^2c_0^2+0.6\eta_i}}$ and $B=c_0\min\{1-\cos\theta_1(0), 1+\cos\theta_2(0)\}/(4\pi)$.
\end{thm}

\begin{remark}
	We may wonder when $\{\eta_n\}_{i=1}^\infty$ satisfy that $\mathbf{w}_i(n)$ would never reach another half-plane separated by $\mathbf{v}$. Here we show that choosing $\eta_n=\Theta(1/n)$ may guarantee this condition. Notice that at this time, $S_n=\Theta(\sqrt{\ln(n)})$. For example of $\mathbf{w}_1(n)$, we have $\cos\theta_1(n) = 1- \mathcal{O}(e^{-B_1S_n})$. Then $\sin\theta_1(n) = \mathcal{O}(e^{-B_1S_n/2})=\mathcal{O}(e^{-c\sqrt{\ln(n)}})$ and $\|\mathbf{w}_1(n)\|=\mathcal{O}(\sqrt{\ln(n)})$. Hence $ \|\mathbf{w}_1(n)\|\sin\theta_1(n) = \mathcal{O}(\sqrt{\ln(n)}e^{-c\sqrt{\ln(n)}})$. We note that $\sqrt{\ln(n)}e^{-c\sqrt{\ln(n)}}\gg 1/n$, showing that $\mathbf{w}_1(n)$ may lie in the same half-plane separated by $\mathbf{v}$ with directional convergence $e^{-\Omega(\sqrt{\ln(n)})}$. 
	Moreover, we can also use learning rate decay in practice. At the beginning, $\|\mathbf{w}_1\|\sin\theta_1(n)$ and $\|\mathbf{w}_2\|\sin\theta_2(n)$ are large, we can use constant learning rate to obtain $S_n = \Theta(\sqrt{n})$, then we decay the learning rate to reduce oscillation around the target direction. 
	However, we consider the condition can be removed if we have elaborate understanding of the remaining initialization.
\end{remark}

\section{Experiments}\label{sec:exp}
\begin{figure*}[t]
	\centering
	\includegraphics[width=0.48\textwidth]{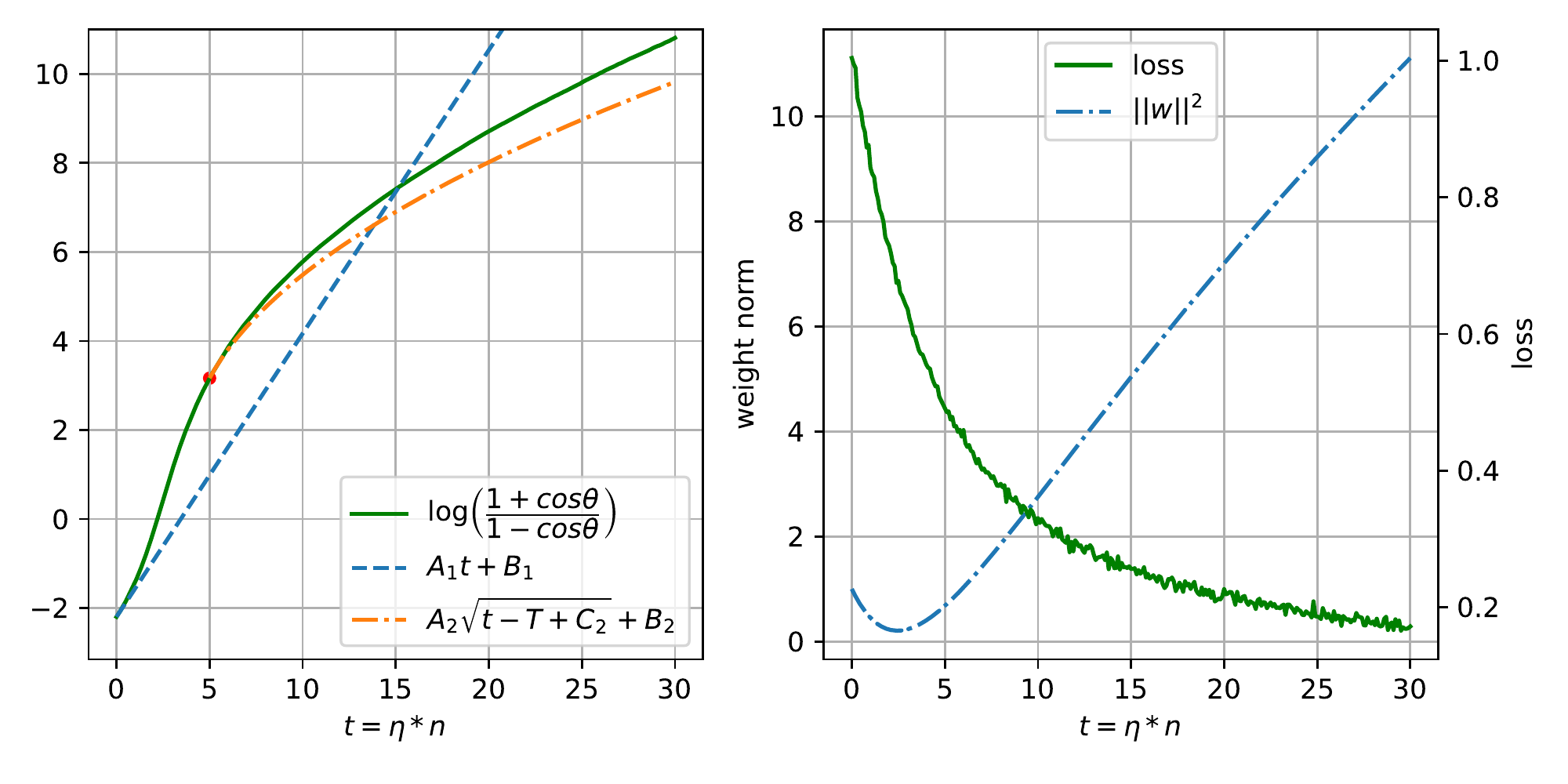}
	\includegraphics[width=0.48\textwidth]{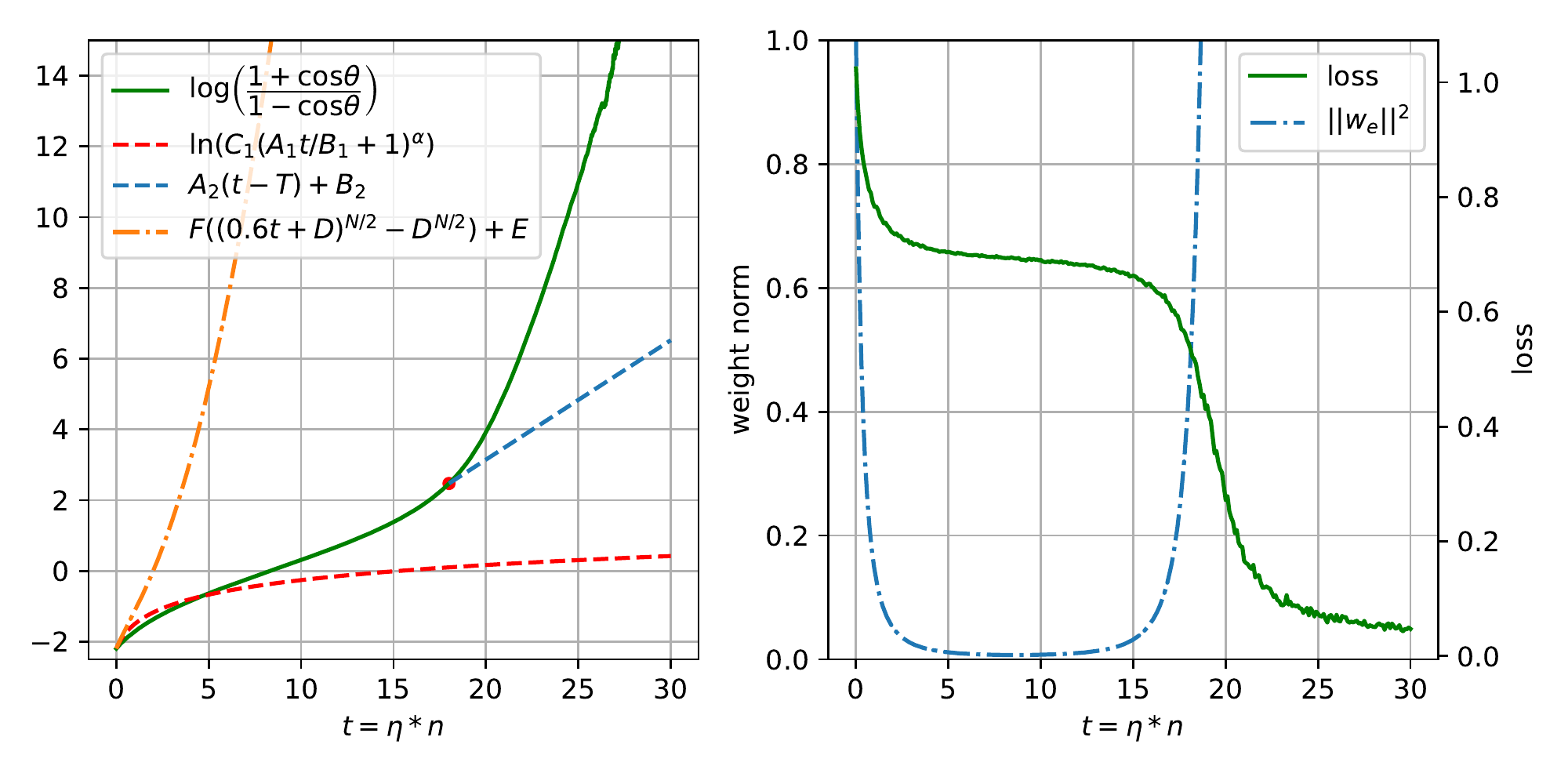}
	\caption{Linear network simulation when $\mathbf{x}\sim\mathcal{U}(\mathcal{S}^1), d=2, \mathbf{v}=(0,1)^\top$. Left two: running SGD with batch size $1000$ from $\mathbf{w}(0)=(0.6,-0.8)^\top$ for $30000$ iterations and constant learning rate $\eta=10^{-3}$. We show our lower bounds in Theorem \ref{thm:linear-flow} at $n{=}0$  and $n{=}5000$. Right two: training 4-layer deep linear network using SGD with batch size $1000$ from $\mathbf{w}_e(0)=(0.6,-0.8)^\top$ for $30000$ iterations and constant learning rate $\eta=10^{-3}$. We also how our lower and upper bounds in Theorem \ref{thm:deep-linear-convergence1} at $n{=}0$ and $n{=}18000$.}
	\label{fig:linear}
\end{figure*}

\begin{figure*}[t]
	\centering
	\includegraphics[width=0.7\linewidth]{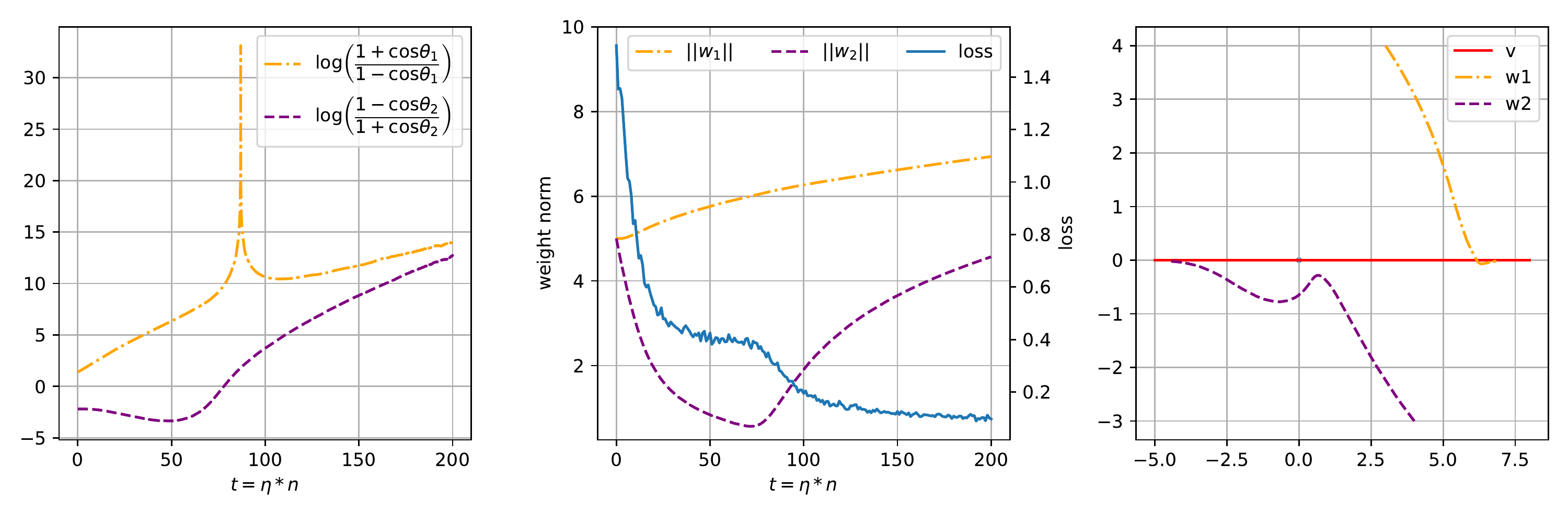}	
	\vspace{-10pt}
	\includegraphics[width=0.7\linewidth]{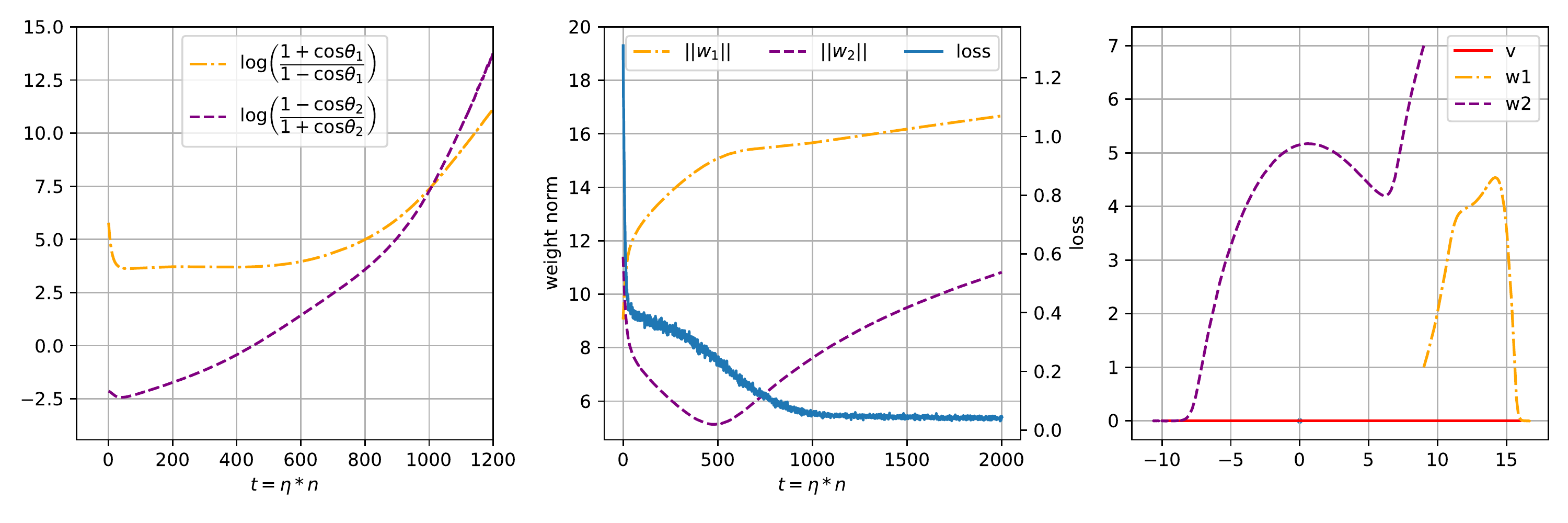}	
	\caption{Two layer network with ReLU activation simulation when $\mathbf{x}\sim\mathcal{U}(\mathcal{S}^1), d=2, \mathbf{v}=(1,0)^\top$. Left three: $\mathbf{w}_1(0)=(3,4)^\top, \mathbf{w}_2(0)=(4, -3)^\top$, which stay in the same half-plane separated by $\mathbf{v}$. We run SGD with batch size $1000$ for $2{\times} 10^4$ iterations and constant learning rate $\eta=10^{-2}$.
		Right three: $\mathbf{w}_1(0)=(9,1)^\top, \mathbf{w}_2(0)=(9,7)^\top$, which stay in the different half-plane separated by $\mathbf{v}$. We run SGD with batch size $1000$ for $2{\times} 10^5$ iterations and constant learning rate $\eta=10^{-2}$. In each experiment, we plot angle variation, loss and weight norm, optimization trajectory in sequence.}
	\label{fig:relu}
	\vspace{-10pt}
\end{figure*}

In this section we conduct some experiments to verify our theoretical analyses. Figure \ref{fig:linear} shows the optimization period for (deep) linear networks. We plot out several lower bound (and upper bound) in our theorems. Although we do not give convergence for stochastic gradient descent, our lower bound in gradient flow still roughly matches the directional convergence in practice, and the weight norm $\|\mathbf{w}(n)\|$ indeed goes through decreasing and increasing period. Moreover, from our analysis, we can see rigorous understanding that the weight norm may give further reasonable results, particularly for deep linear networks.

When non-linear activation added into the decision function, the training dynamic becomes more complicated. We show one example in the left three graphs in Figure \ref{fig:relu}, which satisfies the different plane initialization in Lemma \ref{lemma:ac-angle-grad}. 
Since two weights couple with each other, we can only guarantee (logarithmic) directional convergence for one weight from our proof, such as $\cos\theta_1(n)$ in Figure \ref{fig:relu}. 
Moreover, when applied to (stochastic) gradient descent, we find weight may traverse the target, as the cusp in the first graph in Figure \ref{fig:relu} shows, but the recovery still holds at last. 
The remaining weight may undergo complex period including reversed directional variation at the beginning as $\cos\theta_2(n)$ varies in the third graph of Figure \ref{fig:relu}, but finally the weight still reaches the target direction.
Additionally, we also give a strange trajectory for the remaining initialization that $\mathbf{w}_1(0)$ and $\mathbf{w}_2(0)$ lie in the same half-plane separated by $\mathbf{v}$. 
As depicted by the right three graphs in Figure \ref{fig:relu}, both the weights may encounter totally reverse direction movement at the initial phase, and multi-fluctuation for certain angle (such as $\theta_1(n)$), though both the weights turn into the correct directional improvement soon and finally align with the target as well. Therefore,  it is difficult to derive exact directional convergence rate in this case.  

\section{Conclusion}
In this work, we have studied the behavior of gradient flow and gradient descent on separable data with zero margin under population loss in binary linear classification tasks. 
We have proven logarithmic directional convergence for (deep) linear networks, and shallow non-linear networks with constraint initialization, which is much faster than the general finite sample setting. 
The main mechanism behind our proof comes from the slow increasing of weight norm compared to rapid angle variant towards target. 
Resorting to `lazy training' in previous work, such as over-parameterization, we have found the similar phenomenon between the norm and direction of weights applied to under-parameterization cases for classification tasks under good data distribution as well.
Moreover, we have found that large learning rate (even unbounded) in the early training phase introduces rapid directional improvement for linear activation. However, non-linear activations disturb stable increasing of direction because of the interaction among weights, though the learning rate decay that is broadly used in practice may be seen as a valid substitution.
We hope that our specific view of directional convergence would bring better understanding of the optimization dynamics of gradient methods on neural networks in classification tasks.

\pagebreak
\bibliography{reference}
\bibliographystyle{plainnat}

\appendix
\onecolumn
\section{Missing Proofs in Section \ref{sec:shallow-linear}}\label{app:shallow-linear}
\paragraph{Proofs of Proposition \ref{prop:basic}.}
\begin{proof}
	Lipschitz continuous follows from
	\[ \|\nabla L(\mathbf{w})\| \leq \mathbb{E}_{\mathbf{x}\sim\mathcal{D}} \left\|\frac{y(\mathbf{x})}{1+e^{y(\mathbf{x})\mathbf{w}^{\top}\mathbf{x}}}\mathbf{x} \right\| \leq \mathbb{E}_{\mathbf{x}\sim\mathcal{D}}\left\|\mathbf{x}\right\|\leq \sqrt{tr(\Sigma)}.\]	
	Convexity obtained by
	\[ \nabla^2L(\mathbf{w})=\mathbb{E}_{\mathbf{x}\sim\mathcal{D}} \frac{1}{1+e^{y(\mathbf{x})\mathbf{w}^{\top}\mathbf{x}}} \frac{e^{y(\mathbf{x})\mathbf{w}^{\top}\mathbf{x}}} {1+e^{y(\mathbf{x})\mathbf{w}^{\top}\mathbf{x}}}\mathbf{x} \mathbf{x}^{\top}\succeq \mathbf{0}. \]
	Moreover, for any $\mathbf{u}\in\mathbb{R}^d$ with $\|\mathbf{u}\|=1$,
	\[ \mathbf{u}^\top\nabla^2L(\mathbf{w})\mathbf{u}= \mathbb{E}_{\mathbf{x}\sim\mathcal{D}} \frac{1}{1+e^{y(\mathbf{x})\mathbf{w}^{\top}\mathbf{x}}} \frac{e^{y(\mathbf{x})\mathbf{w}^{\top}\mathbf{x}}} {1+e^{y(\mathbf{x})\mathbf{w}^{\top}\mathbf{x}}}\|\mathbf{x}^{\top}\mathbf{u}\|^2\leq \frac{1}{4} \mathbb{E} \|\mathbf{x}^{\top}\mathbf{u}\|^2 \leq \frac{1}{4}\lambda_{\max}(\Sigma). \]
	However, when $\mathbf{w}=k\mathbf{v}$, and $k \to \infty$ and $P(\{ \mathbf{x}:\mathbf{v}^{\top}\mathbf{x}=0\})=0$,
	\[ \limsup\limits_{k\rightarrow +\infty, \mathbf{w} = k\mathbf{v}} \|\nabla^2 L(\mathbf{w})\| \leq \limsup\limits_{k\rightarrow +\infty, \mathbf{w} = k\mathbf{v}} \mathbb{E}_{\mathbf{x}\sim\mathcal{D}} \frac{1}{1+e^{k|\mathbf{v}^\top \mathbf{x}|}}\frac{e^{k|\mathbf{v}^\top \mathbf{x}|}} {1+e^{k|\mathbf{v}^\top \mathbf{x}|}}\|\mathbf{x}\|^2=0.  \]
	Hence 
	\[ \lim_{k\rightarrow +\infty, \mathbf{w} = k\mathbf{v}}\nabla^2L(\mathbf{w})=\mathbf{0}.  \]
\end{proof}

\paragraph{Proofs of Proposition \ref{prop:grad-norm-prop}.}
\begin{proof}
	Suppose $\|\mathbf{w}\|\leq M$, and notice that 
	\[ -\mathbf{v}^\top\nabla L(\mathbf{w})=\mathbb{E}_{\mathbf{x}\sim\mathcal{D}} \frac{|\mathbf{v}^\top\mathbf{x}|} {1+e^{y(\mathbf{x})\mathbf{w}^{\top}\mathbf{x}}} \geq 0. \]
	\[ -\mathbf{v}^{\top}\nabla L(\mathbf{w})\geq \mathbb{E}_{\mathbf{x}\sim\mathcal{D}} \frac{|\mathbf{v}^\top\mathbf{x}|}{1+e^{\|\mathbf{w}\|\cdot\|\mathbf{x}\|}}\geq \frac{1}{1+e^{MR}}\mathbb{E}|\mathbf{v}^\top\mathbf{x}|\mathbbm{1}_{\{\|\mathbf{x}\| \leq R\}}\geq \frac{r}{1+e^{MR}}P\left(\|\mathbf{x}\| \leq R,|\mathbf{v}^\top\mathbf{x}|\geq r\right). \]
	Let $\epsilon:=P\left(|\mathbf{v}^\top\mathbf{x}|>0\right)>0$, then there exist $R>r>0$, such that $P\left(\|\mathbf{x}\| \leq R,|\mathbf{v}^\top\mathbf{x}|\geq r\right)\geq\epsilon/2$, then $\mathbf{v}^{\top}\nabla L(\mathbf{w}(t))\geq 0.5r\epsilon/(1+e^{MR})>0$.
	
	Hence, $\partial (\mathbf{v}^\top\mathbf{w}(t))/\partial t > 0$ and $\mathbf{v}^\top(\mathbf{w}(n+1)-\mathbf{w}(n)) > 0$, showing that $\mathbf{v}^{\top}\mathbf{w}(t)$ and $\mathbf{v}^{\top}\mathbf{w}(n)$ are increasing. 
	
	If $\|\mathbf{w}(t)\|\leq M$, then $\mathbf{v}^{\top}\mathbf{w}(t)\leq M$. Hence, $\mathbf{v}^{\top}\mathbf{w}(t)$ converges, showing that $\lim_{t\rightarrow \infty}\mathbf{v}^{\top}\nabla L(\mathbf{w}(t))=0$. 
	Using the argument again, we obtain contradiction.
	
	If $\|\mathbf{w}(n)\|\leq M$, then $\mathbf{v}^{\top}\mathbf{w}(n)\leq M$. Hence, $\mathbf{v}^{\top}\mathbf{w}(n)$ converges, showing that $\lim_{n \rightarrow \infty}\eta_n\mathbf{v}^{\top}\nabla L(\mathbf{w}(n))=0$. Since $\eta_n\geq\eta_{-}>0$, then $\lim_{n \rightarrow \infty}\mathbf{v}^{\top}\nabla L(\mathbf{w}(n))=0$, remaining the same argument as the gradient flow case.
\end{proof}

\paragraph{Proofs of Proposition \ref{prop:grad-prop}.}
\begin{proof}
	Since $\mathcal{D}$ is spherically symmetric, we can assume $\mathbf{v}=\mathbf{e}_1$, then 
	\[ \mathbb{E}_{\mathbf{x}\sim\mathcal{D}} \frac{\text{sgn}(\mathbf{v}^{\top} \mathbf{x})}{2} x_i =\frac{1}{2} \mathbb{E}_{\mathbf{x}\sim\mathcal{D}} |x_i| \mathbbm{1}_{\{i=1\}}. \]
	Hence,
	\[ \nabla L(\mathbf{0}) = -\frac{1}{2} \mathbb{E}_{\mathbf{x}\sim\mathcal{D}} |x_1| \mathbf{v}. \]
	In addition,
	\[ \left(-\nabla L(r\mathbf{v})\right)_i = \mathbb{E}_{\mathbf{x}\sim\mathcal{D}}\frac{\text{sgn}(\mathbf{v}^{\top} \mathbf{x})}{1+e^{r|\mathbf{v}^{\top} \mathbf{x}|}} x_i = \mathbb{E}_{\mathbf{x}\sim\mathcal{D}} \frac{|x_1|}{1+e^{|x_1|}}\mathbbm{1}_{\{i=1\}}. \]
	Hence,
	\[  \nabla L(r\mathbf{v}) = -\mathbb{E}_{\mathbf{x}\sim\mathcal{D}} \frac{|x_1|}{1+e^{r|x_1|}}\mathbf{v}. \]
	Finally, 
	\[ \|\nabla L(\mathbf{w})\| = \|\mathbb{E}_{\mathbf{x}\sim\mathcal{D}} \frac{y}{1+e^{y\mathbf{w}^{\top}\mathbf{x}}}\mathbf{x}\| = \|\mathbb{E}_{\mathbf{x}\sim\mathcal{D}_{\mathbf{w},\mathbf{v}}} \frac{y}{1+e^{y\mathbf{w}^{\top}\mathbf{x}}}\mathbf{x}\| \leq \mathbb{E}_{\mathbf{x}\sim\mathcal{D}_{\mathbf{w},\mathbf{v}}} \|\mathbf{x}\| = \mathbb{E}_{\mathbf{x}\sim\mathcal{D}_{2}} \|\mathbf{x}\|=c_0. \]
	The last equality uses the property that $\mathcal{D}$ is spherically symmetric, then $\mathcal{D}_{\mathbf{w},\mathbf{v}}$ is identical distribution for any $\mathbf{w},\mathbf{v}$. We choose the first two dimensions for example.
\end{proof}

\paragraph{Proofs of Lemma \ref{lemma:angle_grad}.}
\begin{proof}
	We denote $ \mathcal{D}_{\mathbf{w}, \mathbf{v}} $ is the marginal distribution of $ \mathbf{x} $ on the $ 2 $-dimensional subspace $ \text{span}\{\mathbf{w},\mathbf{v}\} $, and with a little abuse of notation, we still adopt $ \mathbf{w}, \mathbf{v} \in \mathbb{R}^2 $ as the representations of $\mathbf{w}, \mathbf{v}$ in the subspace $ \text{span}\{\mathbf{w},\mathbf{v}\} $, meaning that we only consider the case in $\mathbb{R}^2$:
	\[ \mathbb{E}_{\mathbf{x}\sim\mathcal{D}_{\mathbf{w}, \mathbf{v}}}\frac{|\mathbf{v}^{\top}\mathbf{x}|-\text{sgn}(\mathbf{v}^{\top}\mathbf{x}) \left(\overline{\mathbf{w}}^{\top} \mathbf{v}\right) \left(\overline{\mathbf{w}}^{\top}\mathbf{x}\right)}{1+e^{\text{sgn}(\mathbf{v}^{\top}\mathbf{x})\mathbf{w}^{\top}\mathbf{x}}}. \]
	Additionally, the expression above is invariant to rotating the coordinate frame, so we can assume without loss of generality
	that $\overline{\mathbf{w}} = (1,0)^{\top}$, $\mathbf{v} = (v_1, v_2)^{\top}$, then we only need to consider
	\[ \mathbb{E}_{\mathbf{x}\sim\mathcal{D}_{\mathbf{w}, \mathbf{v}}}g(\mathbf{x}) := \mathbb{E}_{\mathbf{x}\sim\mathcal{D}_{\mathbf{w}, \mathbf{v}}} \frac{\text{sgn}(\mathbf{v}^{\top}\mathbf{x})\left(v_2x_2\right)} {1+e^{\text{sgn}(\mathbf{v}^{\top}\mathbf{x})\|\mathbf{w}\|x_1}}. \]
	\begin{enumerate}
		\item If $|v_1x_1| > |v_2x_2|$, then $\text{sgn}(\mathbf{v}^{\top}\mathbf{x})=\text{sgn}(v_1x_1)$,
		\[ g(x_1,x_2)+ g(x_1,-x_2)=0. \]
		\item If $|v_2x_2| \geq |v_1x_1|$, then $\text{sgn}(\mathbf{v}^{\top}\mathbf{x})=\text{sgn}(v_2x_2)$,
		\[ g(x_1, x_2)+ g(x_1, -x_2) = \frac{|v_2x_2|}{1+e^{\|\mathbf{w}\|x_1}} +\frac{|v_2x_2|}{1+e^{-\|\mathbf{w}\|x_1}} = |v_2x_2|\geq 0. \]
	\end{enumerate}
	Therefore $\mathbb{E}_{\mathbf{x}\sim\mathcal{D}_{\mathbf{w}, \mathbf{v}}}g(x_1,x_2)\geq 0$. Moreover, if $v_2 \neq0$ and $P(|v_2x_2|\geq|v_1x_1| > 0) > 0$, then $\mathbb{E}_{\mathbf{x}\sim\mathcal{D}}g(x_1,x_2)> 0$. Otherwise, $v_2=0$, then $\mathbb{E}_{\mathbf{x}\sim\mathcal{D}_{\mathbf{w}, \mathbf{v}}}g(x_1,x_2) = 0$, i.e., the angle between $\mathbf{w}$ and $\mathbf{v}$ would not change. \\
	
	More precisely, 
	\begin{equation*}
	\begin{aligned}
	\mathbb{E}_{\mathbf{x}\sim\mathcal{D_{\mathbf{w}, \mathbf{v}}}}g(x_1,x_2) &=\frac{1}{2}\mathbb{E}_{\mathbf{x} \sim\mathcal{D}_{\mathbf{w}, \mathbf{v}}}|v_2x_2|\mathbbm{1}_{\{|v_2x_2| \geq |v_1x_1|\}} \\
	&=\frac{|v_2|}{4\pi} \int_{0}^{\infty}r\int_{|\tan\theta(\mathbf{w}, \mathbf{v}) \cdot \tan\theta|\geq 1}|\sin\theta| d\theta dF(r) \\
	&=\frac{\sin^2\theta(\mathbf{w}, \mathbf{v})}{\pi}\int_{0}^{\infty}r dF(r)\\
	&=\frac{c_0\sin^2\theta(\mathbf{w}, \mathbf{v})}{\pi}.
	\end{aligned}
	\end{equation*}
	where $F(r)$ is the cumulative distribution function of $r := \|\mathbf{x}\|_2$. 

	Thanks to the scope of optimization in the plane $\text{span}\{\mathbf{w}, \mathbf{v}\}$, we can also obtain the projection of $\nabla L(\mathbf{w})$ into $\mathbf{w}$ exactly.

	If $\theta(\mathbf{w}, \mathbf{v})=0$, equality holds, otherwise, from the first result in Lemma \ref{lemma:angle_grad}, we obtain
	\[  -\mathbf{v}^{\top}\left(I-\overline{\mathbf{w}} \; \overline{\mathbf{w}}^{\top}\right) \nabla L(\mathbf{w}) = \frac{c_0}{\pi}\mathbf{v}^{\top} \left(I-\overline{\mathbf{w}} \; \overline{\mathbf{w}}^{\top}\right) \mathbf{v}. \]
	Notice that $\nabla L(\mathbf{w}) \in \text{span}\{\mathbf{w}, \mathbf{v}\}$, and
	\[ \left(I-\overline{\mathbf{w}} \; \overline{\mathbf{w}}^{\top}\right)\left(\frac{c_0}{\pi}\mathbf{v}+\nabla L(\mathbf{w})\right)\perp \mathbf{w}, \mathbf{v}. \]
	Hence
	\[ \left(I-\overline{\mathbf{w}} \; \overline{\mathbf{w}}^{\top}\right)\left(\frac{c_0}{\pi}\mathbf{v}+\nabla L(\mathbf{w})\right)=\mathbf{0}. \]
\end{proof}

\paragraph{Proofs of Lemma \ref{lemma:norm-change}.}
\begin{proof}
	$1)$. The first proposition follows from 
	\[ N(\mathbf{w})=-\mathbf{w}^{\top} \nabla L(\mathbf{w}) = \mathbb{E}_{\mathbf{x}\sim\mathcal{D}} \frac{y(\mathbf{x})\mathbf{w}^{\top}\mathbf{x}}{1+e^{y(\mathbf{x})\mathbf{w}^{\top}\mathbf{x}}}. \]
	Use the fact $x/(1+e^x) < 0.3$, then $N(\mathbf{w})\leq 0.3$. 
	
	Using the spherically symmetric, we may assume that $\overline{\mathbf{w}} = (1,0)^{\top}$ and $\mathbf{v} = (\cos\alpha,\sin\alpha)^{\top}$, then
	
	\[ N(\mathbf{w})= \|\mathbf{w}\| g(\mathbf{x}), \ g(\mathbf{x}):=\mathbb{E}_{\mathbf{x}\sim\mathcal{D_{\mathbf{w},\mathbf{v}}}} \frac{\text{sgn}(\mathbf{v}^{\top}\mathbf{x}) x_1}{1+e^{\|\mathbf{w}\|\text{sgn}(\mathbf{v}^{\top}\mathbf{x})x_1}}.\]
	$2)$ and $3)$. 
	\begin{itemize}
		\item If $|v_1x_1| > |v_2x_2|$, then $\text{sgn}(\mathbf{v}^{\top}\mathbf{x})=\text{sgn}(v_1x_1)$,
		\[ g(x_1, x_2) + g(x_1, -x_2)= \frac{2 \; \text{sgn}(v_1)|x_1|}{1+e^{\|\mathbf{w} \|\text{sgn}(v_1)|x_1|}}. \]
		
		\item If $|v_2x_2| \geq |v_1x_1|$, then $\text{sgn}(\mathbf{v}^{\top}\mathbf{x})=\text{sgn}(v_2x_2)$,
		\[ g(x_1, x_2) + g(x_1, -x_2) = \frac{1-e^{\|\mathbf{w}\|x_1}}{1+e^{\|\mathbf{w}\|x_1}}x_1 \leq 0. \]
	\end{itemize}	
	When $\theta(\mathbf{w}, \mathbf{v}) \geq 0$ and decreases, then the area in the first case with positive contribution becomes larger while the area in the second case becomes smaller, leading to $N(\mathbf{w})$ increasing. As for $\theta(\mathbf{w}, \mathbf{v})\geq \pi/2$, then $v_1\leq 0$ and both cases are negative, showing $N(\mathbf{w})\leq 0$. 
	
	In addition, when $r\leq \|\mathbf{w}\|\leq R$, we have
	\[ N(\mathbf{w})\geq \int_{|v_1x_1| > |v_2x_2|} \frac{2|x_1|}{1+e^{R|x_1|}}d\mathbf{x}-\int_{|v_1x_1| \leq |v_2x_2|} \frac{e^{r|x_1|}-1}{1+e^{r|x_1|}}|x_1|d\mathbf{x}. \]
	Therefore, when $\theta(\mathbf{w},\mathbf{v}) \to 0$, then $v_2\to0, v_1\to 1$, we obtain
	\[ \mathop{\lim\inf}_{\theta(\mathbf{w},\mathbf{v}) \to 0} N(\mathbf{w})\geq \int \frac{2|x_1|}{1+e^{R|x_1|}}d\mathbf{x}>0. \]
	Then there exist $\theta_0$ (related to $r,R$), when $\theta(\mathbf{w},\mathbf{v}) < \theta_0$, $N(\mathbf{w})>0$ holds for any $r\leq \|\mathbf{w}\|\leq R$.
	
	$4)$. We fix $\theta(\mathbf{w}, \mathbf{v})$ or $\overline{\mathbf{w}}$, and consider
	\[ \overline{N}(r) := \frac{N(\mathbf{w})}{\|\mathbf{w}\|} = \mathbb{E}_{\mathbf{x}\sim\mathcal{D}_2} \frac{y(\mathbf{x})\overline{\mathbf{w}}^{\top}\mathbf{x}}{1+e^{y(\mathbf{x}) r \overline{\mathbf{w}}^{\top}\mathbf{x}}}. \]
	
	Then when $r \to 0$, we obtain
	\[ \lim_{r \to 0}\overline{N}(r) = \frac{1}{2} \mathbb{E}_{\mathbf{x}\sim\mathcal{D}_2} y(\mathbf{x})\overline{\mathbf{w}}^\top\mathbf{x}=\frac{c_0|\cos\alpha|}{\pi}. \]
	In addition, 
	\[ \left|\frac{\partial \overline{N}(r)}{\partial r}\right| = \left|\mathbb{E}_{\mathbf{x}\sim\mathcal{D}_2} -\frac{\left(\overline{\mathbf{w}}^{\top}\mathbf{x}\right)^2e^{y(\mathbf{x})r\overline{\mathbf{w}}^{\top}\mathbf{x}}}{\left(1+e^{y(\mathbf{x}) r\overline{\mathbf{w}}^{\top}\mathbf{x}}\right)^2}\right| \leq \frac{1}{4} \mathbb{E}_{\mathbf{x}\sim\mathcal{D}_2} \left(\overline{\mathbf{w}}^{\top}\mathbf{x}\right)^2 \leq \frac{c_0}{4}. \]
	Then $\overline{N}(r)$ is $c_0/4$-Lipschitz continuous, hence when $\|\mathbf{w}\| \leq 2|\cos\alpha|/\pi$, we obtain 
	\[ N(\mathbf{w})\cos\alpha>0. \]
\end{proof}

\paragraph{Proofs of Lemma \ref{lemma:norm-increase-linear}.}
\begin{proof}
	From Lemma \ref{lemma:norm-change} and the condition $\partial \|\mathbf{w}(t)\|^{2}/\partial t=0$, we obtain 
	\[ \frac{\partial \cos\theta(t)}{\partial t} = -\left(\mathbf{v}-\mathbf{v}^\top\overline{\mathbf{w}}(t) \overline{\mathbf{w}}(t)\right)^\top \nabla L(\mathbf{w}(t)) = -\mathbf{v}^\top \nabla L(\mathbf{w}(t)) \geq 0. \]
\end{proof}

\paragraph{Proofs of Theorem \ref{thm:linear-flow}.}
\begin{proof}
	Based on Proposition \ref{prop:grad-norm-prop}, Lemma \ref{lemma:norm-change} and \ref{lemma:norm-increase-linear}, $T$ always exists and $\|\mathbf{w}(t)\|$ first decreases when $t \leq T$, then increases when $t\geq T$. Therefore, when $t \leq T$
	\[ \frac{\partial}{\partial t} \cos(\theta(t)) = \frac{1}{\|\mathbf{w}(t)\|} \frac{c_0\sin^2\theta(t)}{\pi} \geq \frac{1}{\|\mathbf{w}(0)\|} \frac{c_0\sin^2\theta(t)}{\pi}. \]
	We obtain
	\[ \frac{1}{2}\ln\frac{1+\cos\theta(t)}{1-\cos\theta(t)} - \frac{1}{2}\ln\frac{1+\cos\theta(t_0)}{1-\cos\theta(t_0)}\geq \frac{c_0}{\|\mathbf{w}(0)\|\pi}t. \]
	Thus
	\[ \cos\theta(t)\geq 1-\frac{2}{e^{A_1t+B_1}+1}, t\leq T. \]
	When $t\geq T$, based on Lemma \ref{lemma:norm-change},
	\[ \frac{\partial}{\partial t}\|\mathbf{w}(t)\|^2 = 2N(\mathbf{w}(t)) \leq 0.6.\]
	Hence $\|\mathbf{w}(t)\|^2\leq 0.6(t-T)+\|\mathbf{w}(T)\|^2$.
	\[ \frac{\partial}{\partial t} \cos(\theta(t)) = \frac{1}{\|\mathbf{w}(t)\|} \frac{c_0\sin^2\theta(t)}{\pi} \geq \frac{1}{\sqrt{0.6(t-T)+\|\mathbf{w}(T)\|^2}} \frac{c_0\sin^2\theta(t)}{\pi}. \]
	We obtain
	\[ \frac{1}{2}\ln\frac{1+\cos\theta(t)}{1-\cos\theta(t)} - \frac{1}{2}\ln\frac{1+\cos\theta(T)}{1-\cos\theta(T)}\geq \frac{2c_0}{0.6\pi}\left(\sqrt{0.6(t-T)+\|\mathbf{w}(T)\|^2}-\|\mathbf{w}(T)\|\right). \]
	Thus
	\[ \cos\theta(t)\geq 1-\frac{2}{e^{A_2\sqrt{t+C_2}+B_2}+1}. \]
	Note that in the second part, we can choose any $t_0 \geq T$ as the initial point to obtain similar convergence.  
\end{proof}

\paragraph{Proofs of Theorem \ref{thm:negative-angle}.}
First, we need a lemma below:
\begin{lemma}\label{lemma:gd-update}
	\[ \|\mathbf{w}(n+1)\|^2 = \left(\overline{\mathbf{w}}(n)^{\top} \mathbf{w}(n+1)\right)^2 + \left(\frac{c_0\eta_n}{\pi}\right)^2\sin^2\theta(n). \]
\end{lemma}

\begin{proof}
	\begin{equation*}
	\begin{aligned}
	\|\mathbf{w}(n+1)\|^2 &= \|\mathbf{w}(n)\|^2-2\eta_n \mathbf{w}(n)^{\top}\nabla L(\mathbf{w}_n)+\eta_n^2\|\nabla L(\mathbf{w}(n))\|^2 \\ &= \left(\|\mathbf{w}(n)\|- \eta_n\overline{\mathbf{w}}(n)^{\top}\nabla L(\mathbf{w}(n))\right)^2 + \eta_n^2\left\| \left(I-\overline{\mathbf{w}}(n) \; \overline{\mathbf{w}}(n)^{\top}\right) \nabla L(\mathbf{w}(n)) \right\|^2 \\
	&\overset{(1)}{=} \left(\|\mathbf{w}(n)\|- \eta_n\overline{\mathbf{w}}(n)^{\top}\nabla L(\mathbf{w}(n))\right)^2 + \left(\frac{c_0\eta_n}{\pi}\right)^2 \left\|\left(I-\overline{\mathbf{w}}(n) \; \overline{\mathbf{w}}(n)^{\top}\right)\mathbf{v} \right\|^2 \\ 
	&= \left(\|\mathbf{w}(n)\|- \eta_n\overline{\mathbf{w}}(n)^{\top}\nabla L(\mathbf{w}(n))\right)^2 + \left(\frac{c_0\eta_n}{\pi}\right)^2 \sin^2\theta(n) \\
	&= \left( \overline{\mathbf{w}}(n)^{\top} \mathbf{w}(n+1) \right)^2 + \left(\frac{c_0\eta_n}{\pi}\right)^2 \sin^2\theta(n).
	\end{aligned}
	\end{equation*}
	The equality in $(1)$ follows from Lemma \ref{lemma:angle_grad}.
\end{proof}

Now we turn back to the proof of Theorem \ref{thm:negative-angle}.

\begin{proof}
	Then from Lemma \ref{lemma:norm-change}, $-\mathbf{w}(n)^{\top}\nabla L(\mathbf{w}(n))\leq 0$, leading to 
	\[ \|\mathbf{w}(n+1)\|^2 = \|\mathbf{w}(n)\|^2-2\eta_n \mathbf{w}(n)^{\top}\nabla L(\mathbf{w}(n)) + \eta_n^2\|\nabla L(\mathbf{w}(n))\|^2 \leq \|\mathbf{w}(n)\|^2 + \eta_n^2 c_0^2. \]
	
	From Lemma \ref{lemma:gd-update}, $\|\mathbf{w}(n+1)\|\geq\left|\overline{\mathbf{w}}(n)^{\top} \mathbf{w}(n+1)\right| \geq \|\mathbf{w}(n)\|-\eta_n\overline{\mathbf{w}}(n)^{\top}\nabla L(\mathbf{w}(n))$. 
	
	\begin{equation*}
	\begin{aligned}
	\cos\ &\theta(n+1)-\cos\theta(n) = \frac{1}{\|\mathbf{w}(n+1)\|} \left( \mathbf{v}^{\top}\left(\mathbf{w}(n)-\eta_n\nabla L(\mathbf{w}(n))\right)-\|\mathbf{w}(n+1)\| \mathbf{v}^{\top}\overline{\mathbf{w}}(n)\right) \\
	&\geq \frac{1}{\|\mathbf{w}(n+1)\|} \left( \mathbf{v}^{\top}\left(\mathbf{w}(n)-\eta_n\nabla L(\mathbf{w}(n))\right)-\left( \|\mathbf{w}(n)\|-\eta_n \overline{\mathbf{w}}(n)^{\top}\nabla L(\mathbf{w}(n))\right) \mathbf{v}^{\top}\overline{\mathbf{w}}(n)\right) \\ 
	&= \frac{1}{\|\mathbf{w}(n+1)\|}\left( -\eta_n\left( \mathbf{v}-(\overline{\mathbf{w}}(n)^{\top}\mathbf{v}) \overline{\mathbf{w}}(n)\right)^{\top}\nabla L(\mathbf{w}(n))\right) \\
	&= \frac{1}{\|\mathbf{w}(n+1)\|}\frac{c_0 \eta_n \sin^2\theta(n)}{\pi} \\
	&\geq \frac{\eta_n}{\pi\sqrt{A+\sum_{i=0}^{n}\eta_n^2}}\sin^2\theta(n) > 0.
	\end{aligned}
	\end{equation*}
	Thus $\theta(n)$ decrease when $n$ increase. i.e. $\frac{\pi}{2} \leq \theta(n+1) \leq \theta(n)$ and $0 \geq \cos\theta(n+1) \geq \cos\theta(n) \geq \cos\theta(0) > -1$. Then we obtain 
	\begin{equation*}
	\begin{aligned}
	\cos\theta(n+1)-\cos\theta(n) & \geq\frac{\eta_n(1-\cos\theta(n)) (1+\cos\theta(n))}{\pi\sqrt{A+\sum_{i=0}^{n}\eta_n^2}} \geq \left(1-\cos\theta(n)\right)\frac{B\eta_n}{\sqrt{A+\sum_{i=0}^{n}\eta_i^2}}.
	\end{aligned}
	\end{equation*}
	Adjust the terms, we get 
	\[ \left(1-\frac{B\eta_n}{\sqrt{A+\sum_{i=0}^{n}\eta_i^2}}\right)\left(1-\cos\theta(n)\right) \geq 1-\cos\theta(n+1), \]
	showing that
	\[ \ln\left(1-\cos\theta(n)\right) -\ln\left(1-\cos\theta(n+1) \right) 
	\geq -\ln\left(1-\frac{B\eta_n}{\sqrt{A+\sum_{i=0}^{n}\eta_i^2}} \right) \geq \frac{B\eta_n}{\sqrt{A+\sum_{i=0}^{n}\eta_i^2}}. \]
	Hence, 
	\[ \cos \theta(n) \geq 1-\left(1-\cos\theta(0)\right) e^{-BS_n^-}. \]
	Since $S_n^- \to \infty $, then for some finite $T$ steps, $\cos\theta(T) \geq 0$, giving that $\mathbf{v}^\top\mathbf{w}(T)\geq 0$.
\end{proof}

\paragraph{Proofs of Theorem \ref{thm:suff}.}
\begin{proof}
	Use the upper bound of $\|\mathbf{w}(n+1)\|$ below 
	\[ \|\mathbf{w}(n+1)\|^2 = \|\mathbf{w}(n)\|^2-2\eta_n \mathbf{w}(n)^{\top}\nabla L(\mathbf{w}_n) + \eta_n^2\| \nabla L(\mathbf{w}_n) \|^2 \leq \|\mathbf{w}(n)\|^2+0.6\eta_n+c_0^2\eta_n^2. \]
	Therefore, from Lemma \ref{lemma:gd-update},
	\begin{equation*}
	\begin{aligned}
	\cos \; & \theta\left(n{+}1\right){-}\cos\theta\left(n\right)
	= \frac{\mathbf{v}^{\top}\mathbf{w}(n+1)-\|\mathbf{w}(n+1)\|\mathbf{v}^{\top} \overline{\mathbf{w}}(n)}{\|\mathbf{w}(n+1)\|} \\
	&= \frac{\left(\mathbf{v}-\mathbf{v}^{\top}\overline{\mathbf{w}}(n)\overline{\mathbf{w}}(n)\right)^{\top} \mathbf{w}(n+1)-\left(\|\mathbf{w}(n+1)\|-\overline{\mathbf{w}}(n)^{\top} \mathbf{w}(n+1)\right) \mathbf{v}^{\top}\overline{\mathbf{w}}(n)}{\|\mathbf{w}(n+1)\|}\\
	&= \frac{1}{\|\mathbf{w}(n+1)\|}\left({-}\eta_n\left(\mathbf{v}{-}(\overline{\mathbf{w}}(n)^{\top}\mathbf{v}) \overline{\mathbf{w}}(n)\right)^{\top}\nabla L(\mathbf{w}(n)){-}\frac{\|\mathbf{w}(n{+}1)\|^2{-}\left(\overline{\mathbf{w}}(n)^{\top} \mathbf{w}(n{+}1)\right)^2}{\|\mathbf{w}(n{+}1)\|+\overline{\mathbf{w}}(n)^{\top} \mathbf{w}(n{+}1)} \cos\theta(n) \right) \\
	&= \frac{1}{\|\mathbf{w}(n+1)\|}\left(\frac{c_0 \eta_n \sin^2\theta(n)}{\pi}- \frac{c_0^2\eta_n^2\sin^2\theta(n)\cos\theta(n)/\pi^2}{\|\mathbf{w}(n+1)\|+\overline{\mathbf{w}}_n^{\top} \mathbf{w}(n+1)}\right)\\
	&\geq \frac{1}{\|\mathbf{w}(n+1)\|}\frac{\delta c_0 \eta_n \sin^2\theta(n)}{(1+\delta)\pi}\\
	&\geq \frac{\delta \eta_n}{(1+\delta)\pi\sqrt{A+\sum_{i=0}^{n}\eta_i^2+C\eta_i}}\sin^2\theta(n). \\
	\end{aligned}
	\end{equation*}
	Hence, similarly, we obtain 
	\[ \cos\theta(n) \geq 1-\left(1-\cos\theta(0)\right) e^{-BS_n^+}. \]
\end{proof}

\paragraph{Proofs of Theorem \ref{thm:angle-convergence}.}
\begin{proof}
	Set $R_1:=\eta_{+}c_0+\eta_{+}c_0/\pi$, $R_2:=R_1+\eta_{+}c_0$, $\mathcal{T}_1=\{n: \|\mathbf{w}(n)\| < R_1\}, \mathcal{T}_2=\{n: \|\mathbf{w}(n)\| < R_1, \|\mathbf{w}(n+k)\| \geq R_1, \|\mathbf{w}(n+K)\| > R_2, 1\leq k \leq K, 2\leq K\}$.
	Notice that $\|\mathbf{w}(n+1)\|-\|\mathbf{w}(n)\|\leq \eta_{+}c_0$, hence the definition in $\mathcal{T}_2$ is meaningful.
	\begin{enumerate}
		\item If $|\mathcal{T}_1| < \infty$, then there exists $n_0$, such that $\|\mathbf{w}(n)\|\geq R_1$ for all $n\geq n_0$, showing logarithmic convergence from $n_0$.
		\item Otherwise, $|\mathcal{T}_1|=\infty$,  since $\|\mathbf{w}(n)\|$ is unbounded, therefore, $|\mathcal{T}_2|=\infty$. We list $\mathcal{T}_2=\{n_{1}, n_{2}, n_{3}, \dots\}$ and the corresponding sequence length for each $n_i$ is $K_{i}$. 
		Every subsequence $\{n_i{+}1,\dots,n_i{+}K_i\}$ gives linear convergence for $1-\cos\theta(n)$ with rate only related to $c_0, \eta_{+}, R_2$. 
		Additionally,  notice that $\mathbf{v}^{\top}\mathbf{w}(n)=\|\mathbf{w}(n)\|\cos\theta(n)$ is increasing, giving that  $\|\mathbf{w}(n_{i+1}+1)\|\cos\theta(n_{i+1}+1)\geq \|\mathbf{w}(n_i+K_i)\|\cos\theta(n_i+K_i)$.
		From Theorem \ref{thm:negative-angle}, we only need consider $\cos\theta(n) \geq 0$.
		Since $\|\mathbf{w}(n_{i+1})\|\leq R_1$, $\|\mathbf{w}(n_{i+1}+1)\|\leq R_2< \|\mathbf{w}(n_i+K_i)\|$, showing that $\cos\theta(n_{i+1}+1)\geq \cos\theta(n_i+K_i)$. Therefore, combining all sequence $\{n_i{+}1,\dots,n_i{+}K_i\}_{i=1}^{\infty}$ guarantees the directional convergence of gradient descent.
	\end{enumerate}
\end{proof}

\paragraph{Proofs of Corollary \ref{corr:angle-convergence}.}
\begin{proof}
	Based on Theorem \ref{thm:angle-convergence}, there exists $n_0$ such that $\|\mathbf{w}(n_0)\|\cos\theta(n_0)\geq R_1$, so we have $\|\mathbf{w}(n)\|\geq \|\mathbf{w}(n)\|\cos\theta(n)\geq \|\mathbf{w}(n_0)\|\cos\theta(n_0)\geq R_1, \forall \; n>n_0$.
	Hence, from Theorem \ref{thm:suff} and its Remark, we obtain logarithmic convergence of $1-\cos\theta(n)$ from $n_0$.
\end{proof}

\subsection{Initialization Methods under Gradient Descent} \label{app:init}
\begin{enumerate}
	\item Notice that $c_0$ is a constant (dimension-free) only related to the marginal distribution of $\mathcal{D}$, when $\eta_+ < \kappa \frac{\pi}{c_0(1+\pi)}$ for some small constant $\kappa>0$, and $\mathbf{w}(0)\sim\mathcal{N}(0, \tau^2I_d)$.
	\[ P(\mathbf{w}^{\top}(0)\mathbf{v}\geq R_1) \geq 1-\Phi(\kappa/\tau) \geq \frac{1}{2}-\frac{\kappa}{\tau\sqrt{2\pi}}, \]
	where $\Phi(\cdot)$ is the c.d.f. of standard Gaussian distribution. Therefore, if $\eta_+$ is small or $\|\mathbf{w}(0)\|$ is large, then with constant probability, we have logarithmic convergence rate from the initial point.
	\item Set $c_1= \mathbb{E}_{\mathbf{x}\sim\mathcal{D}}|x_1|, c_2= \mathbb{E}_{\mathbf{x}\sim\mathcal{D}}x_1^2$. If $\|\mathbf{w}(0)\| \leq r:=c_1/c_2$, and $\eta_0\geq 4\eta_+(c_0+c_0/\pi)/c_1+4/c_2$, $\eta_+ \geq \eta_n$, we have $\mathbf{v}^\top\mathbf{w}(1)\geq R_1$ from $\lambda_{\max}(\Sigma)=\mathbb{E}_{\mathbf{x}\sim\mathcal{D}} \|\mathbf{x}\|^2/d=\mathbb{E}_{\mathbf{x}\sim\mathcal{D}}x_1^2$, and
	\begin{equation*}
	\begin{aligned}
	\mathbf{v}^\top\mathbf{w}(1) &=\mathbf{v}^\top\left(-\eta_0 \nabla L(\mathbf{0})+\eta_0\left(\nabla L(\mathbf{0})-\nabla L(\mathbf{w}(0))\right)+\mathbf{w}(0)\right) \\
	& \geq \frac{c_1\eta_0}{2}-\frac{\eta_0c_2r}{4}-r \\
	& \geq \frac{c_1\eta_0}{4}-\frac{c_1}{c_2} \\
	&\geq \eta_+(c_0+c_0/\pi)=R_1.
	\end{aligned}
	\end{equation*}
	We can see if we initialize small $\mathbf{w}(0)$, then the condition of convergence can be satisfied after one step update. If data is normalized, i.e., $\mathbf{x} \sim \mathcal{N}(0, I_d)$, then $c_0=\sqrt{\pi/2}, c_1=\sqrt{2/\pi}, c_2=1$, thus we only need $\eta_0 \geq 7\eta_++4$ and $\|\mathbf{w}(0)\|\leq \sqrt{2/\pi}$ to give logarithmic convergence rate from $\mathbf{w}(1)$. If we initialize $\mathbf{w}(0) \sim \mathcal{N}(0, \frac{1}{d\pi}I_d)$, then with high probability that $\mathbf{w}(0)$ satisfies the requirement.
\end{enumerate}

\section{Missing Proofs in Section \ref{sec:deep-linear}}\label{app:deep-linear}
\paragraph{Proofs of Lemma \ref{lemma:deep_linear_grad}.}
\begin{proof}
	From the induced flow on $\mathbf{w}_e$, we obtain
	\begin{equation*}
	\begin{aligned}
	- \left(\mathbf{v}-\left(\overline{\mathbf{w}}_e^{\top} \mathbf{v}\right) \overline{\mathbf{w}}_e\right)^{\top} & \|\mathbf{w}_e\|^{2-\frac{2}{N}} \left(\frac{d L^1(\mathbf{w}_e)}{d\mathbf{w}}+(N-1)\overline{\mathbf{w}}_e \overline{\mathbf{w}}_e^\top \frac{d L^1(\mathbf{w}_e)}{d\mathbf{w}} \right) \\
	&= -\|\mathbf{w}_e\|^{2-\frac{2}{N}} \left(\mathbf{v}-\left(\overline{\mathbf{w}}_e^{\top} \mathbf{v}\right) \overline{\mathbf{w}}_e\right)^{\top} \frac{d L^1(\mathbf{w}_e)}{d\mathbf{w}} \\
	&= \|\mathbf{w}_e\|^{2-\frac{2}{N}} \frac{c_0\sin^2\theta(\mathbf{w}_e, \mathbf{v})}{\pi}.
	\end{aligned}
	\end{equation*}
	Therefore we have 
	\[ \frac{\partial \cos\theta(\mathbf{w}_e(t), \mathbf{v})}{\partial t} = \frac{1}{\|\mathbf{w}_e(t)\|} \left(\mathbf{v}-\left(\overline{\mathbf{w}}_e(t)^{\top} \mathbf{v}\right) \overline{\mathbf{w}}_e(t)\right)^{\top} \frac{\partial \mathbf{w}_e(t)}{\partial t} = \frac{c_0\sin^2\theta(\mathbf{w}_e(t), \mathbf{v})}{\pi}\|\mathbf{w}_e(t)\|^{1-\frac{2}{N}}.\]
\end{proof}

\paragraph{Proofs of Proposition \ref{prop:norm-increase}.}
First, we have induced norm variation flow as below: 
\begin{lemma} \label{lemma:deep_linear_norm} 
	If $\mathbf{w}_e(t) \neq \mathbf{0}$, then
	\[ \frac{\partial \|\mathbf{w}_e(t)\|^{\frac{2}{N}}}{\partial t} = \mathbb{E}_{\mathbf{x}\sim\mathcal{D}_2} \frac{2y\mathbf{w}_e(t)^\top\mathbf{x}}{1+e^{y\mathbf{w}_e(t)^\top\mathbf{x}}}. \]
\end{lemma}
We now turn to prove the Proposition.

\begin{proof}
	\begin{equation*}
	\begin{aligned}
	\frac{\partial \|\mathbf{w}_e(t)\|^{\frac{2}{N}}}{\partial t} &=\frac{2}{N}\|\mathbf{w}_e(t)\|^{\frac{2}{N}-1}\overline{\mathbf{w}}_e(t)^\top\frac{\partial \mathbf{w}_e(t)}{\partial t} \\
	&= -\frac{2}{N} \mathbf{w}_e(t)^\top \left(\frac{d L^1(\mathbf{w}_e(t))}{d\mathbf{w}} +(N-1)\overline{\mathbf{w}}_e(t)\overline{\mathbf{w}}_e(t)^\top \frac{d L^1(\mathbf{w}_e(t))}{d\mathbf{w}} \right)\\
	&= -2\mathbf{w}_e(t)^\top \frac{d L^1(\mathbf{w}_e(t))}{d\mathbf{w}}=\mathbb{E}_{\mathbf{x}\sim\mathcal{D}_2} \frac{2y\mathbf{w}_e(t)^\top\mathbf{x}}{1+e^{y\mathbf{w}_e(t)^\top\mathbf{x}}}.
	\end{aligned}
	\end{equation*}
\end{proof}

\paragraph{Proofs of Proposition \ref{prop:norm-increase}.}
\begin{proof}
	From Lemma \ref{lemma:deep_linear_norm} and the condition $\partial \|\mathbf{w}_e(t)\|^{\frac{2}{N}}/\partial t=0$, we obtain 
	\[ \mathbf{w}_e(t)^\top \frac{d L^1(\mathbf{w}_e(t))}{d\mathbf{w}}=0. \]
	Therefore,
	\[ \frac{\partial \mathbf{w}_e(t)}{\partial t} = -\|\mathbf{w}_e(t)\|^{2-\frac{2}{N}} \left(\frac{d L^1(\mathbf{w}_e(t))}{d\mathbf{w}} +(N-1)\overline{\mathbf{w}}_e(t)\overline{\mathbf{w}}_e(t)^\top \frac{d L^1(\mathbf{w}_e(t))}{d\mathbf{w}} \right) = -\|\mathbf{w}_e(t)\|^{2-\frac{2}{N}} \frac{d L^1(\mathbf{w}_e(t))}{d\mathbf{w}}. \]
	\[ \frac{\partial \cos\theta(\mathbf{w}_e(t), \mathbf{v})}{\partial t} = \frac{1}{\|\mathbf{w}_e(t)\|} \left(\mathbf{v}-\left(\overline{\mathbf{w}}_e(t)^{\top} \mathbf{v}\right) \overline{\mathbf{w}}_e(t)\right)^{\top} \frac{\partial \mathbf{w}_e(t)}{\partial t}=  -\|\mathbf{w}_e(t)\|^{1-\frac{2}{N}}\mathbf{v}^\top\frac{d L^1(\mathbf{w}_e(t))}{d\mathbf{w}}\geq 0. \]
\end{proof}

\paragraph{Proofs of Lemma \ref{lemma:deep-linear-norm}.}
\begin{proof}
	Lower bound comes from 
	\[ \frac{\partial \|\mathbf{w}_e(t)\|^{\frac{2}{N}}}{\partial t} =\mathbb{E}_{\mathbf{x}\sim\mathcal{D}_2} \frac{2y\mathbf{w}_e(t)^\top\mathbf{x}}{1+e^{y\mathbf{w}_e(t)^\top\mathbf{x}}}\geq -2c_0\|\mathbf{w}_e\|. \]
	Integrate both sides from $[0, t]$, we obtain 
	\[ \frac{2}{2-N} \frac{\partial \|\mathbf{w}_e(t)\|^{\frac{2}{N}-1}}{\partial t} \geq -2c_0. \]
	We get that 
	\[
	\|\mathbf{w}_e(t)\|^{1-\frac{2}{N}} \geq \left(\|\mathbf{w}_e(0)\|^{\frac{2}{N}-1} + (N-2)c_0t\right)^{-1}.
	\]
	Upper bound comes from 
	\[ \frac{\partial \|\mathbf{w}_e(t)\|^{\frac{2}{N}}}{\partial t} =\mathbb{E}_{\mathbf{x}\sim\mathcal{D}_2} \frac{2y\mathbf{w}_e(t)^\top\mathbf{x}}{1+e^{y\mathbf{w}_e(t)^\top\mathbf{x}}} \leq 0.6. \]
	\[ \|\mathbf{w}_e(t)\| \leq \left(\|\mathbf{w}_e(0)\|^{\frac{2}{N}} + 0.6t\right)^{\frac{N}{2}}. \]
\end{proof}

\paragraph{Proofs of Theorem \ref{thm:deep-linear-convergence1}.}
Similar as the linear predictor case, we have the monotonicity of the induced norm below: 
\begin{prop}\label{prop:norm-increase}
	If $\mathbf{w}_e(t)\neq \mathbf{0}$ and $\partial \|\mathbf{w}_e(t)\|^2/\partial t=0$, then $\partial \cos\theta(\mathbf{w}_e(t), \mathbf{v})/\partial t \geq 0$.
\end{prop}
\begin{proof}
	\[ 0 = \frac{\mathbf{w}_e(t)^\top\partial \mathbf{w}_e(t)}{\partial t} = -N\|\mathbf{w}_e(t)\|^{2-\frac{2}{N}} \frac{\mathbf{w}_e(t)^\top d L^1(\mathbf{w}_e(t))}{d\mathbf{w}}. \]
	Then 
	\[ \frac{\partial \mathbf{w}_e(t)}{\partial t} = -\|\mathbf{w}_e(t)\|^{2-\frac{2}{N}}\frac{d L^1(\mathbf{w}_e(t))}{d\mathbf{w}} . \]
	Therefore 
	\[ \frac{\partial \cos\theta(\mathbf{w}_e(t), \mathbf{v})}{\partial t} = \left(\mathbf{v}-\mathbf{v}^\top\overline{\mathbf{w}}_e(t) \overline{\mathbf{w}}_e(t)\right)^\top \frac{\partial \mathbf{w}_e(t)}{\partial t} = \frac{\mathbf{v}^\top\partial \mathbf{w}_e(t)}{\partial t} = -\|\mathbf{w}_e(t)\|^{2-\frac{2}{N}}\frac{\mathbf{v}^\top d L^1(\mathbf{w}_e(t))}{d\mathbf{w}}\geq 0. \]
\end{proof}

From Proposition \ref{prop:norm-increase}, once $\mathbf{w}_e(t)$ drop into the norm increasing area $S:=\{\mathbf{w}: \partial \|\mathbf{w}_e(t)\|^2/\partial t \geq 0 \}$, then it would always stay in this area, since in the boundary curve $\partial S$, the angle would become smaller, and using Lemma \ref{lemma:norm-change}, the weight $\mathbf{w}_e(t)$ would still lie in $S$. \\
Now we give the proof of Theorem \ref{thm:deep-linear-convergence1}.

\begin{proof}
	Using Lemma \ref{lemma:deep_linear_grad} and the lower bound of $\|\mathbf{w}_e\|$ in Lemma \ref{lemma:deep-linear-norm}:
	\[ \frac{1}{2}\ln\frac{1+\cos\theta(t)}{1-\cos\theta(t)} - \frac{1}{2}\ln\frac{1+\cos\theta(0)}{1-\cos\theta(0)}\geq \frac{c_0}{\pi}\ln\frac{At+B}{B}. \]
	Finally, we get 
	\[ \cos\theta(t)\geq 1-\frac{2}{C(At/B+1)^\alpha+1}. \]
	Notice that the above convergence holds for the whole training period, hence we always have the directional convergence.
	
	Now we show the existence of $T$ when $\partial \|\mathbf{w}_e(0)\|^2/\partial t < 0 $. If $\partial \|\mathbf{w}_e(t)\|^2/\partial t < 0 $ all the time, then $\|\mathbf{w}_e(t)\|$ converges.
	
	\begin{itemize}
		\item If $\|\mathbf{w}_e(t)\| \to 0$, we would never converge to origin, which is a saddle point. The reason follows from the above result that we always have the directional convergence, then $\cos\theta(t)$ would always increase and become positive at some time $t_0>0$ with $\cos\theta(t) \geq \delta>0, \, \forall \, t\geq t_0$. Therefore if $\mathbf{w}_e(t)$ converges to origin, then from Property 4 in Lemma \ref{lemma:norm-change}, for small enough $\|\mathbf{w}_e(t_1)\|$, we have $N(\mathbf{w}_e(t_1))\geq 0$, thus $ \partial \|\mathbf{w}_e(t)\|^2/\partial t\geq 0$ would always holds for $t\geq t_1$, showing contradiction for $\|\mathbf{w}_e(t)\| \to 0$.
		\item If $\|\mathbf{w}_e(t)\| \to \epsilon > 0$, then we have $\|\mathbf{w}(0)\| \geq \|\mathbf{w}_e(t)\| \geq \epsilon $, since $\cos\theta(t)\to 1$, then from Property 3 in Lemma \ref{lemma:norm-change}, we would obtain $\partial \|\mathbf{w}_e(t)\|^2/\partial t >0 $, contradiction!
	\end{itemize} 
	Therefore, there exists some time $T$ such that $\partial \|\mathbf{w}_e(T)\|^2/\partial t = 0 $.
	
	For the second part, using Proposition \ref{prop:norm-increase}, we have 
	\[ \frac{\partial \|\mathbf{w}_e(t)\|^{\frac{2}{N}}}{\partial t} \geq 0, \ \forall t \geq T. \]
	Hence $\|\mathbf{w}_e(t)\|\geq \|\mathbf{w}_e(T)\| $. Therefore,
	\[ \frac{1}{2}\ln\frac{1+\cos\theta(t)}{1-\cos\theta(t)} - \frac{1}{2} \ln\frac{1+\cos\theta(T)}{1-\cos\theta(T)}\geq \frac{c_0}{\pi}\|\mathbf{w}_e(T)\|^{2-\frac{2}{N}}(t-T). \]
	\[ \cos\theta(t)\geq 1-\frac{2}{1+e^{A(t-T)+B}}, t\geq T. \]
	For the third part, using the upper bound of $\|\mathbf{w}_e\|$ in Lemma \ref{lemma:deep-linear-norm}:
	\[ \frac{1}{2}\ln\frac{1+\cos\theta(t)}{1-\cos\theta(t)} - \frac{1}{2}\ln\frac{1+\cos\theta(0)}{1-\cos\theta(0)}\geq \frac{2c_0}{0.6N\pi}\left(\left(0.6t+D\right)^{\frac{N}{2}}-D^{\frac{N}{2}}\right). \]
	Finally, we obtain 
	\[ \cos\theta(t)\geq 1-\frac{2}{e^{F[\left(0.6t+D\right)^{N/2}-D^{N/2}]+E}+1}. \]
\end{proof}

\section{Missing Proofs in Section \ref{sec:shallow-non-linear}}\label{app:shallow-nonlinear}
\paragraph{Proofs of Proposition \ref{prop:grad-update}.}
\begin{proof}
	Let $\epsilon = 1-P\left(\mathbf{x}=\mathbf{0}\right)>0$, we can choose $M, m > 0$, such that $P(0<\|\mathbf{x}\|\leq M, |\mathbf{v}^\top\mathbf{x}| \geq m)\geq \epsilon/2$. Denote $\|\mathbf{w}_1\| \leq R_1, \|\mathbf{w}_2\| \leq R_2$, $M_1 = \sup_{|t|\leq R_1M}\sigma(t)$, $M_2 = \sup_{|t|\leq R_2M}\sigma(t)$, then
	$ |y(\mathbf{x})\phi(\mathbf{x})| \leq |\sigma(\mathbf{w}_1^\top\mathbf{x})| + |\sigma(\mathbf{w}_2^\top\mathbf{x})| \leq M_1+M_2 $.
	\[ -\mathbf{v}^\top\nabla_{1}L(\mathbf{w}) = \mathbb{E}_{\mathbf{x}\sim\mathcal{D}} \frac{|\mathbf{v}^\top\mathbf{x}|\sigma'(\mathbf{w}^{\top}_1\mathbf{x})}{1+e^{y(\mathbf{x})\phi(\mathbf{x})}} \geq \frac{0.5m\gamma(R_1M)}{1+e^{M_1+M_2}}P(0<\|\mathbf{x}\|\leq M, |\mathbf{v}^\top\mathbf{x}| \geq m) \geq \frac{0.25m\epsilon\gamma(R_1M)}{1+e^{M_1+M_2}}>0. \]
	
	In addition, same argument holds for $\mathbf{v}^\top\nabla_{2}L(\mathbf{w})$.

	From previous argument, We already obtain $\mathbf{v}^\top\mathbf{w}_1(t)$ and $\mathbf{v}^\top\mathbf{w}_1(n)$ are increasing, and $\mathbf{v}^\top\mathbf{w}_2(t)$ and $\mathbf{v}^\top\mathbf{w}_2(n)$ are decreasing. If $\|\mathbf{w}_1(t)\|$ and $\|\mathbf{w}_2(t)\|$ (or $\|\mathbf{w}_1(n)\|$ and $\|\mathbf{w}_2(n)\|$) are bounded, then $\mathbf{v}^\top\mathbf{w}_1(t)$ (or $\mathbf{v}^\top\mathbf{w}_1(n)$) converges. We obtain contradiction similar to the proof of Proposition \ref{prop:grad-update} by :
	\[ \mathbf{v}^\top \nabla_{1} L(\mathbf{w}(t)) \to 0, \; \quad \; \eta_n\mathbf{v}^\top \nabla_{1} L(\mathbf{w}(n)) \to 0. \]
\end{proof}

\begin{corollary}\label{cor:angle-inv}
	When $\cos\theta_1(t_0)\geq0$ for some $t_0\geq 0$, then $\forall t>t_0$, $\cos\theta_1(t)>0$. Similarly, when $\cos\theta_2(t_0)\leq0$ for some $t_0\geq 0$, then $\forall t>t_0$, $\cos\theta_1(t)<0$.
\end{corollary}

\begin{prop}\label{prop:para}
	For $i=1,2$, if $\mathbf{w}_i\neq\mathbf{0}$ and $\mathbf{w}_i \parallel \mathbf{v}$, then $\left(\mathbf{v}-(\overline{\mathbf{w}}_i^\top\mathbf{v})\overline{\mathbf{w}}_i\right)^\top\nabla_{i} L(\mathbf{w}) = 0$.
\end{prop}
\begin{proof}
	If $\mathbf{w}_i\neq\mathbf{0}$ and $\mathbf{w}_i \parallel \mathbf{v}$, then $\mathbf{v}-(\overline{\mathbf{w}}_i^\top\mathbf{v})\overline{\mathbf{w}}_i=0$.
\end{proof}

\paragraph{Proofs of Lemma \ref{lemma:ac-angle-grad}.}
\begin{proof}
	Since $\mathbf{x}\sim\mathcal{D}$ has a spherically symmetric distribution, and $\mathbf{w}_1$, $\mathbf{w}_2$, $\mathbf{v}$ are in the same plane. We only need consider the dynamic in $\mathbb{R}^2$. In addition, w.l.o.g., we can assume $\mathbf{w}_1=(r_1, 0)^\top$, $\mathbf{v}=(v_1,v_2)^\top=(\cos\alpha, \sin\alpha)^\top$, $\mathbf{w}_2=(a, b)^\top$, then 	
	\[\left(\mathbf{v}-(\overline{\mathbf{w}}_1^\top\mathbf{v})\overline{\mathbf{w}}_1\right)^\top\left(\mathbf{w}_2-(\overline{\mathbf{w}}_1^\top\mathbf{w}_2)\overline{\mathbf{w}}_1\right)=v_2b.\]
	\begin{equation*}
	\begin{aligned}
	-\left(\mathbf{v}-(\overline{\mathbf{w}}_1^\top\mathbf{v})\overline{\mathbf{w}}_1\right)^\top\nabla_{1} L(\mathbf{w}) &= \mathbb{E}_{\mathbf{x}\sim\mathcal{D}} \frac{y(\mathbf{\mathbf{x}})\sigma'(\mathbf{w}^{\top}_1\mathbf{x})\left(\mathbf{v}-(\overline{\mathbf{w}}_1^\top\mathbf{v})\overline{\mathbf{w}}_1\right)\mathbf{x}}{1+e^{y(\mathbf{x})\phi(\mathbf{x})}} \\
	&= \mathbb{E}_{\mathbf{x}\sim\mathcal{D}_2} \frac{y(\mathbf{x})\sigma'(r_1x_1)v_2x_2}{1+e^{y(\mathbf{x})\left(\sigma(r_1x_1)-\sigma(ax_1+bx_2)\right)}} 
	\end{aligned}
	\end{equation*}
	Set $g(x_1, x_2) := \dfrac{y(\mathbf{x})\sigma'(r_1x_1)v_2x_2}{1+e^{y(\mathbf{x})\left(\sigma(r_1x_1)-\sigma(ax_1+bx_2)\right)}} $. Referred to the proof in linear case, we need to consider 
	\[ \mathbb{E}_{\mathbf{x}\sim\mathcal{D}_2} g(x_1, x_2) = \int_{|v_1x_1| \geq |v_2x_2|}g(x_1, x_2)d\mathbf{x} + \int_{0\leq |v_1|x_1 < |v_2x_2|} g(x_1, x_2) d\mathbf{x} \triangleq I_1 + I_2. \]
	\begin{enumerate}
		\item When $|v_1x_1| > |v_2x_2|$, 
		\begin{equation*}
		\begin{aligned}
		g(x_1, x_2) + g(x_1, -x_2) &= \frac{\text{sgn}(v_1x_1)\sigma'(r_1x_1)v_2x_2} {1+e^{\text{sgn}(v_1x_1) \left(\sigma(r_1x_1)-\sigma(ax_1+bx_2)\right)}}-\frac{\text{sgn}(v_1x_1) \sigma'(r_1x_1)v_2x_2} {1+e^{\text{sgn}(v_1x_1)\left(\sigma(r_1x_1)-\sigma(ax_1-bx_2) \right)}} \\
		&= \frac{\sigma'(r_1x_1)\text{sgn}(v_1x_1)v_2x_2 e^{\text{sgn}(v_1x_1)\sigma(r_1x_1)} \left(e^{-\text{sgn}(v_1x_1)\sigma(ax_1-bx_2)}-e^{-\text{sgn}(v_1x_1)\sigma(ax_1+bx_2)}\right)
		} {\left(1+e^{\text{sgn}(v_1x_1)\left(\sigma(r_1x_1)-\sigma(ax_1+bx_2)\right)}\right) \left(1+e^{\text{sgn}(v_1x_1)\left(\sigma(r_1x_1)-\sigma(ax_1-bx_2)\right)}\right)}.
		\end{aligned}
		\end{equation*}
		\[ \text{sgn}(g(x_1, x_2) + g(x_1, -x_2)) = \text{sgn}(v_1x_1v_2x_2 \cdot v_1x_1bx_2) = \text{sgn}(v_2b). \]
		Hence when $ v_2b \geq 0$, $g(x_1, x_2) + g(x_1, -x_2) \geq 0$.
		
		\item When $|v_1x_1| \leq |v_2x_2|$,
		\begin{equation*}
		\begin{aligned}
		g(x_1, x_2) + g(x_1, -x_2) &= \frac{|v_2x_2|\sigma'(r_1x_1)} {1+e^{\text{sgn}(v_2x_2)\left(\sigma(r_1x_1)-\sigma(ax_1+bx_2)\right)}}+ \frac{|v_2x_2|\sigma'(r_1x_1)} {1+e^{-\text{sgn}(v_2x_2)\left(\sigma(r_1x_1)-\sigma(ax_1-bx_2)\right)}}.
		\end{aligned}
		\end{equation*}
		When $ v_2b \geq 0$,
		\[y\left(\sigma(r_1x_1)-\sigma(ax_1+bx_2)\right)-y\left(\sigma(r_1x_1)-\sigma(ax_1-bx_2)\right)=y\left(\sigma(ax_1-bx_2)-\sigma(ax_1+bx_2)\right)\leq 0. \]
		Combining with the fact that
		\[ x+y \leq 0, \frac{1}{1+e^x}+\frac{1}{1+e^y}\geq 1. \]
		Therefore, $g(x_1, x_2) + g(x_1, -x_2) \geq |v_2x_2|\sigma'(r_1x_1)$. 
	\end{enumerate}
	From the above discussion, 
	\begin{equation*}
	\begin{aligned}
	\mathbb{E}_{\mathbf{x}\sim\mathcal{D}_2} g(x_1, x_2) &\geq \frac{1}{2}\int_{|v_1x_1| \leq |v_2x_2|} |v_2x_2|\sigma'(r_1x_1) d\mathbf{x} \\
	&= \frac{|\sin\alpha|}{4\pi}\int_{0}^{\infty} \int_{|\tan\theta\tan\alpha|\geq 1} |r\sin\theta|\sigma'(r_1r\cos\theta) d\theta dF(r) \\ &= \frac{|\sin\alpha|}{2\pi}\int_{0}^{\infty} \int_{\pi/2 + \tilde{\alpha} \geq \theta \geq \pi/2-\tilde{\alpha}} r\sin\theta\sigma'(r_1r\cos\theta) d\theta dF(r) \\
	&= \frac{\sin^2\alpha}{2\pi} \int_{0}^{\infty} \frac{1}{r_1|\sin\alpha|}\left(\sigma(r_1r|\sin\alpha|)-\sigma(-r_1r|\sin\alpha|)\right)dF(r) \\
	&= \frac{\sin^2\alpha}{2\pi} \nu(\|\mathbf{w}_1\sin\alpha\|), \\
	\end{aligned}
	\end{equation*}
	Therefore,
	\[ -\left(\mathbf{v}-(\overline{\mathbf{w}}_1^\top\mathbf{v})\overline{\mathbf{w}}_1\right)^\top\nabla_{1} L(\mathbf{w}) \geq \frac{\nu(\|\mathbf{w}_1\sin\theta_1\|)}{2\pi} \sin^2 \theta_1. \]
	When applied to $\mathbf{w}_2$, notice that 
	$L(\mathbf{w}_2, \mathbf{w}_1, -v) = L(\mathbf{w}_1, \mathbf{w}_2, v) =\mathbb{E}_{\mathbf{x}\sim\mathcal{D}}\ln(1+e^{sgn(\mathbf{v}^\top\mathbf{x})\left( \sigma(\mathbf{w}_1^\top\mathbf{x})-\sigma(\mathbf{w}_2^\top\mathbf{x}) \right)})$, we can view $\mathbf{w}_2$ as $\mathbf{w}_1$ when applied with target direction $-\mathbf{v}$. Using the fact $\sin\theta_2=\sin(\pi-\theta_2)$, we obtain same results for $\mathbf{w}_2$.
\end{proof}

\paragraph{Proofs of Proposition \ref{prop:unbounded}.}
\begin{proof}
	We only need to show the case that only one unbounded weight is impossible. Suppose $\|\mathbf{w}_2(t)\|\leq R_2$, but $\|\mathbf{w}_1(t)\|$ is unbounded and we consider $\mathbf{w}_1(t) \neq \mathbf{0}$ in the following analysis. 
	Let $\epsilon = 1-P\left(\mathbf{v}^\top\mathbf{x}=\mathbf{0}\right)>0$,  then $P(\mathbf{v}^\top\mathbf{x} \cdot \mathbf{w}_1^\top\mathbf{x} < 0, \mathbf{w}_2^\top\mathbf{x} > 0) = 0.5 P( \mathbf{v}^\top\mathbf{x} \cdot \mathbf{w}_1^\top\mathbf{x} < 0) =\theta_1\epsilon/(2\pi)$. Therefore, we can choose $M, m(\theta_1) > 0$ with $\lim_{\theta\to 0}m(\theta) = 0$, such that $P(\|\mathbf{x}\|\leq M, |\mathbf{v}^\top\mathbf{x}| \geq m, \mathbf{v}^\top\mathbf{x} \cdot \mathbf{w}_1^\top\mathbf{x} < 0, \mathbf{w}_2^\top\mathbf{x} > 0) \geq \theta_1\epsilon/(4\pi)$.
	Meanwhile, when $\mathbf{v}^\top\mathbf{x} \cdot \mathbf{w}_1^\top\mathbf{x} < 0$, $y(\mathbf{x}) \cdot \sigma(\mathbf{w}_1^\top \mathbf{x}) < \sigma(0)$. Denote $M_2 = \sup_{|t|\leq R_2M}\sigma(t)$, $\mathcal{T} = \|\mathbf{x}\|\leq M, |\mathbf{v}^\top\mathbf{x}| \geq m, \mathbf{v}^\top\mathbf{x} \cdot \mathbf{w}_1^\top\mathbf{x} < 0, \mathbf{w}_2^\top\mathbf{x} > 0$. We finally obtain 
	\[ \mathbf{v}^\top\nabla_{2}L(\mathbf{w}) = \mathbb{E}_{\mathbf{x}\sim\mathcal{D}} \frac{|\mathbf{v}^\top\mathbf{x}|\sigma'(\mathbf{w}^{\top}_2\mathbf{x})}{1+e^{y(\mathbf{x})\phi(\mathbf{x})}} \geq \mathbb{E}_{\mathbf{x}\in\mathcal{T}} \frac{|\mathbf{v}^\top\mathbf{x}| \sigma'(\mathbf{w}^{\top}_2\mathbf{x})}{1+e^{y(\mathbf{x})\phi(\mathbf{x})}} \geq 
	\frac{m(\theta_1)\gamma(R_2M)\theta_1\epsilon}{4\pi\left(1+e^{\sigma(0)+M_2}\right)} > 0. \]
	Therefore $\mathbf{v}^\top\mathbf{w}_2(t)$ can't converge unless $\theta_1(t) \to 0$. However, when $\mathbf{w}_1(t) \to r_1(t)\mathbf{v}, r_1(t)>0$,
	\[ \mathbf{v}^\top\nabla_{2}L(\mathbf{w}) \to\mathbb{E}_{\mathbf{x}\sim\mathcal{D}} \frac{|\mathbf{v}^\top\mathbf{x}|\sigma'(\mathbf{w}^{\top}_2\mathbf{x})}{1+e^{y(\mathbf{x})\left( \sigma(r_1(t)\mathbf{v}^\top\mathbf{x})-\sigma(\mathbf{w}_2^\top\mathbf{x}) \right) }} \geq   \mathbb{E}_{\mathbf{v}^\top\mathbf{x}<0<\mathbf{w}_2^\top\mathbf{x}} \frac{|\mathbf{v}^\top\mathbf{x}| \sigma'(\mathbf{w}^{\top}_2\mathbf{x})}{1+e^{y(\mathbf{x})\phi(\mathbf{x})}} = \mathbb{E}_{\mathbf{v}^\top\mathbf{x}<0<\mathbf{w}_2^\top\mathbf{x}} \frac{|\mathbf{v}^\top\mathbf{x}|}{1+e^{\mathbf{w}^{\top}_2\mathbf{x}}}. \]
	Then $\mathbf{w}_2(t)\to r_2(t)\mathbf{v}, r_2(t)>0$. Unfortunately,
	\[ \mathbf{v}^\top\nabla_{2}L(\mathbf{w}) \to \mathbb{E}_{\mathbf{v}^\top\mathbf{x}>0} \frac{|\mathbf{v}^\top\mathbf{x}| }{1+e^{(r_1(t)-r_2(t))|\mathbf{v}^\top\mathbf{x}|}}\leftarrow -\mathbf{v}^\top\nabla_{1}L(\mathbf{w}). \]
	Then $\partial r_1(t) / \partial t + \partial r_2(t) / \partial t \to 0$, then the variation of $r_1(t)$ and $r_2(t)$ would become same finally, but $r_1(t) \to \infty$. Similar argument holds when applied to gradient descent.
\end{proof}

\paragraph{Proofs of Lemma \ref{lemma:relu-norm}.}
\begin{proof}
	\begin{equation*}
	\begin{aligned}
	\frac{\partial \left(\|\mathbf{w}_1(t)\|^2+\|\mathbf{w}_2(t)\|^2\right)}{\partial t} &= 2\mathbb{E}_{\mathbf{x}\sim\mathcal{D}} \frac{y(\mathbf{\mathbf{x}}) \left(\mathbf{w}_1^{\top}\mathbf{x} \sigma'(\mathbf{w}^{\top}_1 \mathbf{x})-\mathbf{w}_2^{\top}\mathbf{x}\sigma'(\mathbf{w}^{\top}_1 \mathbf{x})\right)}{1+e^{y(\mathbf{x})\phi(\mathbf{x})}} \\
	&= 2\mathbb{E}_{\mathbf{x}\sim\mathcal{D}} \frac{y(\mathbf{\mathbf{x}}) \left(\sigma(\mathbf{w}^{\top}_1 \mathbf{x})-\sigma(\mathbf{w}^{\top}_2 \mathbf{x})\right)}{1+e^{y(\mathbf{x})\phi(\mathbf{x})}} \\
	&= 2\mathbb{E}_{\mathbf{x}\sim\mathcal{D}} \frac{y(\mathbf{\mathbf{x}}) \phi(\mathbf{x})}{1+e^{y(\mathbf{x})\phi(\mathbf{x})}} \leq 0.6. \\
	\end{aligned}
	\end{equation*}
\end{proof}

\paragraph{Proofs of Theorem \ref{thm:diff-init}.}
\begin{proof}
	For ReLU activation, $\nu(z) = \mathbb{E}_{\mathbf{x} \sim\mathcal{D}_2} \left(\sigma(z\|\mathbf{x}\|)-\sigma(-z\|\mathbf{x}\|) \right) / z = \mathbb{E}_{\mathbf{x} \sim\mathcal{D}_2} \|\mathbf{x}\|=c_0$.
	Based on Proposition \ref{prop:para}, $\mathbf{w}_1$ and $\mathbf{w}_2$ would always settle in the different half-plane separated by  $\mathbf{v}$.
	Set $\mathcal{T} := \{t: \cos\theta_1(t)+\cos\theta_2(t)=0 \ \text{and} \ \partial \left(\cos\theta_1(t) + \cos\theta_2(t)\right) / \partial t \neq 0\}$. Since $\cos\theta_1(t)+\cos\theta_2(t)$ is 
	continuous related to $t$, and any $t\in\mathcal{T}$, there exists a $\epsilon_t>0$ such that $\cos\theta_1(t_s)+\cos\theta_2(t_s)\neq 0, \forall t_s\in (t-\epsilon_t, t+\epsilon_t)$, therefore, $\mathcal{T}$ is a countable set, and $\mathcal{T}=\{t_1,t_2\dots\}$, $t_0=0$.
	
	We first prove if $\mathcal{T} \neq \emptyset$. For each $t_i \in \mathcal{T}$, if $\cos\theta_1(t_i)+\cos\theta_2(t_i)$ reverses from positive to negative at time $t_i$. Then $t_{i-1}\leq t \leq t_i$, $\mathbf{w}_2(t)$ and $\mathbf{v}$ lie in the same half-plane separated by $\mathbf{w}_1(t)$, then
	\[ \frac{\partial \cos\theta_1(t)}{\partial t} \geq \frac{c_0}{2\pi\|\mathbf{w}_1(t)\|}\sin^2\theta_1(t). \]
	Based on Lemma \ref{lemma:relu-norm}, 
	\[ \|\mathbf{w}_i(t)\|^2 \leq \|\mathbf{w}_1(t)\|^2+\|\mathbf{w}_2(t)\|^2\leq 0.6 t+\|\mathbf{w}_1(0)\|^2+\|\mathbf{w}_2(0)\|^2, \ i=1, 2. \]
	Set $r(t)^2 := \|\mathbf{w}_1(t)\|^2+\|\mathbf{w}_2(t)\|^2$, then,
	\[ \frac{\partial \cos\theta_1(t)}{\partial t} \geq \frac{c_0\sin^2 \theta_1(t)}{2\pi\sqrt{0.6 t+r^2}} . \]
	We obtain
	\begin{equation}\label{eq:update1}
	-\ln\tan\frac{\theta_1(t)}{2}+\ln\tan\frac{\theta_1(t_{i-1})}{2}\geq \frac{c_0}{0.6\pi} \left(\sqrt{0.6t+r(0)^2}-\sqrt{0.6t_{i-1}+r(0)^2}\right). 
	\end{equation}
	Similarly, if $\cos\theta_1(t_i)+\cos\theta_2(t_i)$ reverses from negative to positive at time $t_i$, then
	\begin{equation}\label{eq:update2}
	-\ln\tan\frac{\pi-\theta_2(t)}{2}+\ln\tan\frac{\pi-\theta_2(t_{i-1})}{2} \geq \frac{c_0}{0.6\pi} \left(\sqrt{0.6t+r(0)^2}-\sqrt{0.6t_{i-1}+r(0)^2}\right). 
	\end{equation}
	Since $\theta_2(t_i)+\theta_1(t_i)=\pi, \forall i\geq 1$, then Eq. \ref{eq:update1} or \ref{eq:update2} holds for $i\geq 1$. Therefore, we obtain for any $t>0$, summing over all $[t_{i-1}, t_i], t_{i}<t$, 
	\[ -\ln\tan\frac{\theta_i(t)}{2}+\ln\tan\frac{\max\{\theta_1(0),\pi-\theta_2(0)\}}{2} \geq \frac{c_0}{0.6\pi} \left(\sqrt{0.6t+r(0)^2}-r(0)\right), i=1 \ or \ 2. \] 
	\[ (-1)^{i-1}\cos\theta_i(t) \geq 1-\frac{2}{e^{A\sqrt{t+B}+C}+1}, \]
	with $A = 2c_0/(\pi\sqrt{0.6}), B=r(0)^2/0.6$, and $C=-2c_0r(0)/(0.6\pi)+2\ln\tan\frac{\max\{\theta_1(0),\pi-\theta_2(0)\}}{2}$, and the choice of $i$ depends on the the position of weights at $t$. 
	
	Now we consider $\mathcal{T}=\emptyset$, and suppose $\cos\mathbf{w}_1(0) + \cos\mathbf{w}_2(0) \geq 0$. Then $\mathbf{w}_2(t)$ and $\mathbf{v}$ always lie in the same half-plane separated by $\mathbf{w}_1(t)$. Hence, we still have 
	\[ \cos\theta_1(t) \geq 1-\frac{2}{e^{A\sqrt{t+B}+C}+1}. \]
\end{proof}

\paragraph{Proofs of Theorem \ref{thm:diff-init2}.}
\begin{proof}
	From Proposition \ref{prop:grad-update}, $\mathbf{w}_2(t)^\top\mathbf{v}$ is decreasing for any $\mathbf{w}_1(t)$ and $\mathbf{w}_2(t)$, and the origin is not a stationary point, thus we always have $\mathbf{w}_2(t_0)^\top\mathbf{v} \leq 0$, for some $t_0\geq 0$.
	
	Now from Corollary \ref{cor:angle-inv} and Proposition \ref{prop:para}, $\mathbf{w}_2(t)^\top\mathbf{v}\leq 0, \ \forall t>t_0$.
	When $\mathbf{w}_1 = r_1\mathbf{v}, r_1 \geq 0$, in the following expectation, using spherically symmetric distribution assumption, we assume $\mathbf{w}_2=(r_2,0)^\top$, then $\mathbf{v}=(\cos\theta_2, \sin\theta_2)^\top$, then 
	\begin{equation*}
	\begin{aligned}
	\left(\mathbf{v}-(\overline{\mathbf{w}}_2^\top\mathbf{v})\overline{\mathbf{w}}_2\right)^\top\nabla_{2} L(\mathbf{w}) &= \mathbb{E}_{\mathbf{x}\sim\mathcal{D}_2, x_1>0} \frac{y(\mathbf{x})v_2x_2}{1+e^{y(\mathbf{x}) \left(\sigma(\mathbf{w}_1^\top\mathbf{x})-r_2x_1\right)}} \\
	&=\left(\mathbb{E}_{\mathbf{v}^\top\mathbf{x}\geq 0, x_1\geq 0} \frac{v_2x_2}{1+e^{ r_1\mathbf{v}^\top\mathbf{x}-r_2x_1}}-\mathbb{E}_{\mathbf{v}^\top\mathbf{x}\leq 0, x_1\geq 0} \frac{v_2x_2}{1+e^{r_2x_1}}\right) \\
	&\overset{(1)}{\geq} - \mathbb{E}_{\mathbf{v}^\top\mathbf{x}\leq 0, x_1\geq 0} \frac{v_2x_2}{1+e^{r_2x_1}} \\
	&=-\frac{1}{2\pi}\int_{0}^\infty \int_{-\pi/2}^{\alpha-\pi/2}\frac{v_2r\sin\theta}{1+e^{rr_2\cos\theta}}d\theta dF(r)\\
	&=-\frac{\sin\theta_2}{2\pi r_2} \int_{0}^\infty \ln\frac{1+e^{-rr_2\sin\theta_2}}{2} dF(r) \geq 0.
	\end{aligned}
	\end{equation*}
	The inequality in $(1)$ holds from $v_1\leq 0$ due to $\mathbf{w}_2(t)^\top\mathbf{v} \leq 0$ and $v_2x_2\geq v_1x_1+v_2x_2\geq 0$. Moreover, once $r_1 \gg r_2$, the inequality would become tight.
	
	We denote `$\gtrsim$' as greater when multiplied and added by some constant, then invoking from the above analysis, we discover that when $r_2(t)\sin\theta_2(t)$ is large, 
	\[ -\frac{\partial\cos\theta_2(t)}{\partial t} \gtrsim \frac{\sin\theta_2(t)}{2\pi r_2^2} \gtrsim \frac{\sin\theta_2(t)}{t} \ \Rightarrow \ \theta_2(t) \gtrsim \ln t \Rightarrow \ 1+\cos\theta_2(t) \lesssim -\Theta(\ln t). \]
	
	When $r_2(t)\sin\theta_2(t)$ is small, then 
	\[ -\frac{\partial\cos\theta_2(t)}{\partial t} \gtrsim \frac{\sin^2\theta_2(t)}{2\pi r_2} \ \Rightarrow \ -\cos\theta_2(t) \gtrsim 1-e^{-\Theta(\sqrt{t})}. \]
	Due to the effect of $\mathbf{w}_1(t)$ and the norm of $\mathbf{w}_2(t)$, $\mathbf{w}_2$ may go through complex directional training dynamic between $-\Theta(\ln t) \sim e^{-\Theta(\sqrt{t})}$ for $1+\cos\theta_2(t)$.
\end{proof}

\paragraph{Proofs of Theorem \ref{thm:same-init}.}
\begin{proof}
	Notice that when $\mathbf{w}_1(0)$ and $\mathbf{w}_2(0)$ lie in the same half-plane separated by $\mathbf{v}$ and $\theta_1(0) > \theta_2(0)$, we get that the conditions in Lemma \ref{lemma:ac-angle-grad} are satisfied both for $\mathbf{w}_1$ and $\mathbf{w}_2$, thus $\theta_2(t)$ increases and $\theta_1(t)$ decreases. In addition, for $\theta_1(t)$, we have 
	\[ \frac{\partial \cos\theta_1(t)}{\partial t} \geq \frac{c_0}{2\pi \|\mathbf{w}_1\|} \sin^2 \theta_1(t).  \]
	Based on Lemma \ref{lemma:relu-norm}, 
	\[ \|\mathbf{w}_1(t)\|^2+\|\mathbf{w}_2(t)\|^2 \leq 0.6 t+\|\mathbf{w}_1(0)\|^2+\|\mathbf{w}_2(0)\|^2. \]
	Hence,
	\[ \frac{\partial \cos\theta_1(t)}{\partial t} \geq \frac{c_0}{2\pi\sqrt{0.6t+r_0^2}} \sin^2 \theta_1(t). \]
	We obtain
	\[ \cos\theta_1(t)\geq 1-\frac{2}{e^{A\sqrt{t+B}+C}+1}, \]
	for some constant $A, B$ which related to initialization. Similarly, $\theta_2(t)$ has the same convergence rate.
	Therefore, $\theta_1(t_o)=\theta_2(t_o)$ for some $t_0 = \mathcal{O}\left(\ln^2\left(\frac{1}{ \cos\theta_2(0)-\cos\theta_1(0)}\right)\right)$.
\end{proof}

\paragraph{Proofs of Theorem \ref{thm:gd-relu}.}
\begin{proof}
	Motivated from linear network setting, We show for the analysis for $\mathbf{w}_1$ when $\mathbf{w}_2$ and $\mathbf{v}$ lie in the same half-plane separated by $\mathbf{w}_1$. 
	
	For the first step, we need to guarantee positive improvement towards target direction $\mathbf{v}$ (or $-\mathbf{v}$). Notice that 
	\begin{equation*}
	\begin{aligned}
	\|\mathbf{w}_1(n+1)\|^2 &= \|\mathbf{w}_1(n)\|^2-2\eta_n \mathbf{w}_1(n)^{\top}\nabla_{1} L(\mathbf{w}(n)) + 
	\eta_n^2\|\nabla_{1} L(\mathbf{w}(n)) \|^2 \\ 
	&= \left(\|\mathbf{w}_1(n)\|- \eta_n\overline{\mathbf{w}}_1(n)^{\top}\nabla_{1} L(\mathbf{w}(n)) \right)^2 + \eta_n^2\left\| \left(I-\overline{\mathbf{w}}_1(n) \; \overline{\mathbf{w}}_1(n)^{\top}\right) \nabla_{1} L(\mathbf{w}(n)) \right\|^2 \\
	&= \left( \overline{\mathbf{w}}_1(n)^{\top} \mathbf{w}_1(n+1) \right)^2 + \eta_n^2\left\| \left(I-\overline{\mathbf{w}}_1(n) \; \overline{\mathbf{w}}_1(n)^{\top}\right) \nabla_{1} L(\mathbf{w}(n)) \right\|^2.
	\end{aligned}
	\end{equation*}
	We denote $a\overline{\mathbf{w}_1(n)^\perp} := \left(I-\overline{\mathbf{w}}_1(n) \; \overline{\mathbf{w}}_1(n)^{\top}\right) \mathbf{v}$, $b\overline{\mathbf{w}_1(n)^\perp} := -\left(I-\overline{\mathbf{w}}_1(n) \; \overline{\mathbf{w}}_1(n)^{\top}\right) \nabla_{1} L(\mathbf{w}(n))$, $\alpha_i(n)=\theta(\mathbf{w}_i(n+1), \mathbf{w}_i(n))$.
	Then $a^2=\sin^2\theta_1(n)$, and from Lemma \ref{lemma:ac-angle-grad}, $ab\geq c_0\sin^2\theta_1(n) / (2\pi)$. Then we can simply the angle difference between two steps: 
	\begin{equation*}
	\begin{aligned}
	\cos \; & \theta_1(n{+}1){-}\cos\theta_1(n)
	= \frac{\mathbf{v}^{\top}\mathbf{w}_1(n+1)-\|\mathbf{w}_1(n+1)\|\mathbf{v}^{\top} \overline{\mathbf{w}}_1(n)}{\|\mathbf{w}_1(n+1)\|} \\
	&= \frac{\left(\mathbf{v}-\mathbf{v}^{\top}\overline{\mathbf{w}}_1(n)\overline{\mathbf{w}}_1(n)\right)^{\top} \mathbf{w}_1(n+1)-\left(\|\mathbf{w}_1(n+1)\|-\overline{\mathbf{w}}_1(n)^{\top} \mathbf{w}_1(n+1)\right) \mathbf{v}^{\top}\overline{\mathbf{w}}_1(n)}{\|\mathbf{w}_1(n+1)\|} \\
	&= \frac{1}{\|\mathbf{w}_1(n+1)\|}\left({-}\eta_n\left(\mathbf{v}{-}(\overline{\mathbf{w}}_1(n)^{\top}\mathbf{v}) \overline{\mathbf{w}}_1(n)\right)^{\top}\nabla_1 L(\mathbf{w}(n)){-}\frac{\|\mathbf{w}_1(n{+}1)\|^2{-}\left(\overline{\mathbf{w}}_1(n)^{\top} \mathbf{w}_1(n{+}1)\right)^2}{\|\mathbf{w}_1(n{+}1)\|+\overline{\mathbf{w}}_1(n)^{\top} \mathbf{w}_1(n{+}1)} \cos\theta_1(n) \right) \\
	&=\frac{1}{\|\mathbf{w}_1(n+1)\|}\left(\eta_nab-\frac{\eta^2_n b^2\cos\theta_1(n)}{\|\mathbf{w}_1(n+1)\|+\overline{\mathbf{w}}_1(n)^{\top} \mathbf{w}_1(n+1)}\right).
	\end{aligned}
	\end{equation*}
	If $\theta_1(n)\geq \pi/2$, we obtain
	\[ \cos\theta_1(n{+}1){-}\cos\theta_1(n) \geq \frac{1}{\|\mathbf{w}_1(n+1)\|}\frac{c_0 \eta_n \sin^2\theta_1(n)}{2\pi}.\]
	Otherwise, $\theta_1(n) < \pi/2$. In order to derive 
	\[ \cos\theta_1(n{+}1){-}\cos\theta_1(n) \geq \frac{1}{\|\mathbf{w}_1(n+1)\|}\frac{c_0 \delta \eta_n \sin^2\theta_1(n)}{2(1+\delta)\pi}, \]
	we need to find some large enough $\delta>0$ to satisfy
	\begin{equation}\label{eq:suff2}
	\|\mathbf{w}_1(n+1)\|+\overline{\mathbf{w}}_1(n)^{\top} \mathbf{w}_1(n+1) \geq \frac{(1+\delta)\eta_n b\cos\theta_1(n)}{a}.
	\end{equation}
	Notice the following relation:
	\[ \|\mathbf{w}_1(n+1)\|+\overline{\mathbf{w}}_1(n)^{\top} \mathbf{w}_1(n+1) = \|\mathbf{w}_1(n+1)\|(1+\cos\alpha_i(n))\]
	\[ \frac{\eta_n b\cos\theta_1(n)}{a} = \sin\alpha_1(n)\|\mathbf{w}_1(n+1)\|/\tan\theta_1(n). \]
	Therefore, Eq. \ref{eq:suff2} is equivalent to
	\[ \tan\theta_1(n) \geq \frac{(1+\delta)\sin\alpha_1(n)}{1+\cos\alpha_1(n)}=(1+\delta)\tan\frac{\alpha_1(n)}{2}. \]
	From the condition that $\mathbf{w}_1(n)$ would never reach another half-plane separated by $\mathbf{v}$, we already have $\alpha_1(n)\leq\theta_1(n)<\pi/2$. Hence we can deduce $\tan\theta_1(n)\geq \tan\alpha_1(n)$. 
	Taking 
	\[ 1+\delta = \tan\theta_1(n)/\tan\frac{\alpha_1(n)}{2} \geq \tan\alpha_1(n)/\tan\frac{\alpha_1(n)}{2} = \frac{2}{1-\tan^2\frac{\alpha_1(n)}{2}} \geq 2,\] 
	which satisfies the requirement, and showing that 	
	\[ \cos\theta_1(n{+}1){-}\cos\theta_1(n) \geq \frac{1}{\|\mathbf{w}_1(n+1)\|}\frac{c_0 \eta_n \sin^2\theta_1(n)}{4\pi}, \]
	
	Second, obtaining the upper bound of $\|\mathbf{w}_1(t)\|^2+\|\mathbf{w}_2(t)\|^2$ using Lemma \ref{lemma:relu-norm},	
	\begin{equation*}
	\begin{aligned}
	\|\mathbf{w}_1(n+1)\|^2 + \|\mathbf{w}_2(n+1)\|^2 &= \|\mathbf{w}_1(n)\|^2 + \|\mathbf{w}_2(n)\|^2 - 2\eta_n \mathbf{w}_1(n)^{\top}\nabla_{1} L(\mathbf{w}(n)) - 2 \eta_n \mathbf{w}_2(n)^{\top}\nabla_{2} L(\mathbf{w}(n)) \\
	& \quad +\eta_n^2\|\nabla_{1} L(\mathbf{w}(n)) \|^2 + \eta_n^2 \|\nabla_{2} L(\mathbf{w}(n)) \|^2 \\
	&\leq \|\mathbf{w}_1(n)\|^2 + \|\mathbf{w}_2(n)\|^2 + 0.6 \eta_n + 2c_0^2\eta_n^2.
	\end{aligned}
	\end{equation*}
	We obtain the improvement for angle $\theta_1(n)$: if $\left(\mathbf{v}-(\overline{\mathbf{w}}_1^\top\mathbf{v}) \overline{\mathbf{w}}_1\right)^\top\left(\mathbf{w}_2-(\overline{\mathbf{w}}_1^\top\mathbf{w}_2)\overline{\mathbf{w}}_1\right)\geq 0$, then
	\begin{equation}\label{eq:gd-improve}
	\cos\theta_1(n{+}1){-}\cos\theta_1(n) \geq \frac{1}{\|\mathbf{w}_1(n+1)\|}\frac{c_0 \eta_n \sin^2\theta_1(n)}{4\pi} \geq \frac{B\eta_n \left(1-\cos\theta_1(n)\right)}{\sqrt{\|\mathbf{w}_1(0)\|^2 + \|\mathbf{w}_2(0)\|^2 +\sum_{i=0}^{n} 0.6 \eta_i +2c_0^2\eta_i^2}}.
	\end{equation}
	Similar argument holds for $\mathbf{w}_2$ if $\mathbf{w}_1$ and $\mathbf{v}$ lie in the different half-plane separated by $\mathbf{w}_2$: if \\ $\left(\mathbf{v}-(\overline{\mathbf{w}}_2^\top\mathbf{v}) \overline{\mathbf{w}}_2\right)^\top\left(\mathbf{w}_1-(\overline{\mathbf{w}}_2^\top\mathbf{w}_1)\overline{\mathbf{w}}_2\right)\leq 0$, then
	\begin{equation}\label{eq:gd-improve2}
	-\cos\theta_2(n{+}1){+}\cos\theta_2(n) \geq \frac{1}{\|\mathbf{w}_2(n+1)\|}\frac{c_0 \eta_n \sin^2\theta_2(n)}{4\pi} \geq \frac{B\eta_n \left(1+\cos\theta_2(n)\right)}{\sqrt{\|\mathbf{w}_1(0)\|^2 + \|\mathbf{w}_2(0)\|^2 +\sum_{i=0}^{n} 0.6 \eta_i +2c_0^2\eta_i^2}}.
	\end{equation}
	Finally, using the same technique proof in Theorem \ref{thm:diff-init} for the interlaced steps 
	\[ \mathcal{T} := \{n: \left(\cos\theta_1(n)+\cos\theta_2(n)\right)\left( \cos\theta_1(n+1)+\cos\theta_2(n+1)\right)<0 \}. \]
	We always have at least one weight satisfies the improvement in Eq. \ref{eq:gd-improve} or \ref{eq:gd-improve2}. Hence, we obtain the following one satisfies
	\[ \cos\theta_1(n) \geq 1-(1-\cos\theta_1(0))e^{-BS_n}, -\cos\theta_2(n) \geq 1-(1+\cos\theta_2(0))e^{-BS_n}. \]
\end{proof}

\section{Additional Results}\label{app:add-res}
Addition lemmas for future work:
\begin{lemma}[Lemma \ref{lemma:norm-change} improved version]
	If $\mathbf{x}\sim\mathcal{U}\left(\mathcal{S}^{1}\right)$ (uniform distribution in the unit circle), if $N(\mathbf{w})>0$ and $\|\mathbf{w}\|\sin \theta(\mathbf{w}, \mathbf{v}) \geq 2$, then $\|\mathbf{w}\|\sin \theta(\mathbf{w}, \mathbf{v}) \leq 2 \ln\left(\frac{4\pi}{\theta(\mathbf{w}, \mathbf{v})}+1\right)$.
\end{lemma}

\begin{proof}
	We assume $\mathbf{w}=(r,0)^\top, \mathbf{v}=(\cos\alpha,\sin\alpha)$, then $0\leq \alpha\leq \pi/2$ since $N(\mathbf{w})>0$. We obtain
	\begin{equation*}
	\begin{aligned}
	N(\mathbf{w})&=\frac{1}{2\pi}\int_{|v_1x_1| \geq |v_2x_2|}\frac{|x_1|}{1+e^{r|x_1|}} d\theta + \frac{1}{2\pi}\int_{|v_1x_1| \leq v_2x_2}\frac{1-e^{rx_1}}{1+e^{rx_1}}x_1 d\theta \\
	& = \frac{1}{\pi}\int_{0 \leq \theta \leq \pi/2-\alpha}\frac{2\cos\theta}{1+e^{r\cos\theta}} d\theta - \frac{1}{\pi}\int_{\pi/2 \geq \theta \geq \pi/2-\alpha} \frac{e^{r\cos\theta}-1}{e^{r\cos\theta}+1}\cos\theta d\theta \\
	& \leq \frac{1}{\pi}\int_{r\cos\theta\geq r\sin\alpha\geq 2} \frac{2\cos\theta}{e^{r\cos\theta}+1} d\theta -\frac{1}{\pi}\int_{\pi/2-\alpha/2 \geq \theta \geq \pi/2-\alpha} \frac{e^{r\cos\theta}-1}{e^{r\cos\theta}+1} \cos\theta d\theta \\
	& \leq \frac{\pi/2-\alpha}{\pi}\frac{2\sin\alpha}{1+e^{r\sin\alpha}} -\frac{\alpha}{2\pi} \frac{e^{r\sin(\alpha/2)}-1}{e^{r\sin(\alpha/2)}+1} \sin\frac{\alpha}{2} \\
	\end{aligned}
	\end{equation*}
	Hence
	\[ \frac{2\pi\sin\alpha}{1+e^{r\sin\alpha}}\geq \frac{e^{r\sin(\alpha/2)}-1}{e^{r\sin(\alpha/2)}+1} \cdot \alpha\sin\frac{\alpha}{2}. \]
	Using the fact $\sin\alpha\geq\sin\frac{\alpha}{2} \geq \frac{1}{2}\sin\alpha$
	We obtain
	\[ 4\pi \geq 4\pi\cos\frac{\alpha}{2} \geq \alpha \left(1+e^{r\sin\alpha}\right)\frac{e^{r\sin(\alpha/2)}-1}{e^{r\sin(\alpha/2)}+1} \geq \alpha \left(e^{r\sin(\alpha/2)}-1\right). \]
	\[ r\sin\alpha \leq 2r\sin\frac{\alpha}{2}\leq 2\ln\left(\frac{4\pi}{\alpha}+1\right). \]
\end{proof}

\begin{prop}
	$\mathbf{w}_1, \mathbf{w}_2$ are in the same half-plane separated by $\mathbf{v}$, and $\theta_2 \geq \theta_1 + \pi/2$, then
	\[ \frac{\partial \cos\theta_1(t)}{\partial t}\geq 0, \ \text{if} \ \sin(\theta_2(t)-\theta_1(t)) \leq \frac{\ln\left(2e^{r_1(t)\sin\theta_1(t)}-1\right)}{r_1(t)}. \]
	\[ \frac{\partial \cos\theta_2(t)}{\partial t} \leq 0, \ \text{if} \ \sin(\theta_2(t)-\theta_1(t)) \leq \frac{\ln\left(2e^{r_2(t)\sin\theta_2(t)}-1\right)}{r_2(t)}. \]
\end{prop}

\begin{proof} [Sketch]
	When $\theta_2 \geq \theta_1 + \pi/2$ and $\theta_1, \theta_2$ are fixed, we can show that for any $r_1, r_2>0$, 
	\[ \frac{\partial^2 \cos\theta_1(t)}{\partial t\partial r_2} \leq 0, \ \frac{\partial^2 \cos\theta_2(t)}{\partial t\partial r_1} \geq 0.\]
	Then we consider $r_1, r_2 \to \infty$ to obtain the minimum.
\end{proof}

\end{document}